\documentclass[11pt]{article}

\usepackage[final]{acl}

\usepackage{times}
\usepackage{latexsym}

\usepackage[T1]{fontenc}

\usepackage[utf8]{inputenc}

\usepackage{microtype}

\usepackage{algorithm}
\usepackage{algorithmic}

\usepackage{mathtools}

\usepackage{amsmath}    
\usepackage{amssymb}    
\usepackage{amsthm}     

\usepackage{enumitem}

\usepackage{booktabs}       
\usepackage{multirow}
\usepackage{float}          

\usepackage[noabbrev,capitalise]{cleveref}

\newtheorem{theorem}{Theorem}
\newtheorem{proposition}[theorem]{Proposition}
\theoremstyle{remark}
\newtheorem{remark}[theorem]{Remark}

\usepackage{inconsolata}

\usepackage{graphicx}

\title{REIGN: Refurbished Embeddings with Integrated Guidance Networks for Efficient Context-Length Scaling}

\author{
 \textbf{Devrim Çavuşoğlu\textsuperscript{1,2}} \quad
 \textbf{Emre Akbaş\textsuperscript{1}}
\\
\\
 Middle East Technical University\textsuperscript{1},
 OBSS AI\textsuperscript{2}
}

\begin{document}
\maketitle
\begin{abstract}
Dense retrieval over long documents is expensive. Token-level
encoders scale quadratically in sequence length, and most long-context
embedding models reach 32K tokens only through architectural workarounds
or by stretching billion-parameter LLMs. We propose \textbf{REIGN}
(Refurbished Embeddings with Integrated Guidance Networks), a
contrastively trained bi-encoder that operates on sequences of
contextualised chunk embeddings from a frozen Guidance Network (GN)
rather than on raw tokens. REIGN targets multi-chunk inputs,
primarily for document-to-document retrieval; single-chunk inputs
stay with the GN. Decoupling token-level processing from
document-level reasoning, and caching the GN embeddings to disk, cuts
per-document training cost by roughly four orders of magnitude relative
to chunked Transformer fine-tuning. We also release a synthetic
long-document retrieval benchmark for contrastive training and
evaluation at long context lengths. Across an in-distribution Wikipedia
benchmark, the LoCo out-of-distribution suite, and a real-world patent
retrieval case study, REIGN matches dense long-context retrievers at
smaller parameter budgets in each regime. A paired significance test puts
it on par with models $1.6$--$4.3\times$ larger on the patent task, and
it stays within 0.65 nDCG@10 of a $20\times$-larger model on LoCo.
Our code is available at \url{https://github.com/devrimcavusoglu/reign}.
\end{abstract}

\section{Introduction}

Retrieval-augmented generation, multi-hop question answering, and document-level semantic search all need embeddings for inputs that exceed the context window of a standard Transformer. Efficient-attention variants \citep{beltagy2020longformer, zaheer2020big} and memory-augmented architectures extend the window, but embedding documents that run from tens of thousands to millions of tokens still hits memory and compute walls in practice.

Recent long-context embedding models (e.g., Nomic-Embed, \citealp{nussbaum2024nomic}; Jina-Embeddings-v3, \citealp{sturua2024jina}; M2-BERT, \citealp{fu2023monarch}) reach 8K--32K tokens via RoPE extensions, sparse attention, or structured sequence mixing. A parallel line instruction-tunes billion-scale LLMs into embedding models \citep{lee2025gemini, choi2024linqembedmistraltechnicalreport, zhang2025qwen3embeddingadvancingtext, SFRAIResearch2024}; these are accurate but computationally heavy, and beyond most groups' training budgets even with parameter-efficient adapters or quantisation. Both lines still take the sub-word token as the base unit of representation.

Data poses a complementary challenge. Existing embedding benchmarks target short passages and lack the semantically aligned long-document pairs and realistic hard negatives required for contrastive learning at scale. We address this with a synthetic long-document retrieval dataset derived from Wikipedia, in which LLM-generated rephrasals plus semantically injected distractors provide positive and graded negative pairs over documents averaging several thousand words (\Cref{sec:dataset}).

We propose \textbf{REIGN} (\textit{Refurbished Embeddings with Integrated Guidance Networks}), a contrastively trained bi-encoder for long-document retrieval. REIGN takes ToBERT's chunk-and-aggregate intuition~\citep{pappagari2019hierarchical} and combines it with SimCLR-style contrastive training~\citep{chen2020simple}. The one shift is the chunk encoder, a frozen pretrained text embedder (a \textit{Guidance Network}, GN; e.g., GTE or E5) whose outputs are cached. REIGN therefore (i)~scales to document-level inputs without changing the underlying architecture, and (ii)~trains over precomputed chunk embeddings, replacing repeated token-level forward passes with a single small Transformer over a short embedding sequence.

We evaluate REIGN on three retrieval regimes: the in-distribution synthetic Wikipedia benchmark, the out-of-distribution LoCo long-context suite \citep{saad2024benchmarking}, and a real-world patent prior-art case study \citep{ayaou2025dapfam}. A 55M-parameter REIGN+GTE-small beats every native long-context dense baseline we test on GoodWiki-Long and the LoCo-paper zero-shot baselines on LoCo, and a 357M REIGN+GTE-large is statistically indistinguishable from dense baselines $1.6$--$4.3\times$ larger on the patent task.

\section{Related Work}

We situate REIGN with respect to prior work along three axes: (i)~representation granularity, (ii)~long-context modelling, and (iii)~efficiency-oriented training.

\subsection{Tokenisation and Representation Granularity}

Subword tokenisers such as WordPiece \citep{wu2016google} and BPE \citep{sennrich2015neural} are the default input layer for Transformer encoders, but at the document scale the resulting token sequences make the attention cost prohibitive.

Modern embedding models such as E5 \citep{wang2022text}, GTE \citep{li2023towards}, and the systems on MTEB \citep{muennighoff2022mteb} still treat the subword token as the atomic input. REIGN instead treats the chunk-level embedding from a frozen Guidance Network (GN) as the atomic input, and trains only on top of those embeddings.

\subsection{Long-Context Modelling Strategies}

Hierarchical encoders split documents into chunks, encode each chunk independently, and aggregate. ToBERT \citep{pappagari2019hierarchical} and CogLTX \citep{ding2020cogltx} use this recipe for supervised classification over documents of a few hundred to a few thousand words. PARADE \citep{parade2023li} is the closest precedent for REIGN architecturally. It aggregates passage-level representations from a BERT-style encoder with a small Transformer, but it does so as a query--document cross-encoder for supervised re-ranking. SMITH \citep{yang2020smith}, a Siamese hierarchical encoder for document-to-document matching, is the closest in task. Its sentence- and document-level Transformers are trained jointly, need dedicated masked-LM pretraining, and cap inputs at 2{,}048 tokens. Where these systems train the hierarchy end-to-end, REIGN freezes the chunk encoder and caches its outputs, trains a bi-encoder under a contrastive objective for first-stage retrieval, and targets documents several times longer on average, with tails into the hundreds of thousands of tokens.

A separate line extends the Transformer itself to longer contexts via sparse or structured attention (Longformer and LED \citep{beltagy2020longformer}, BigBird \citep{zaheer2020big}) or via sub-quadratic structured-matrix backbones, as in M2-BERT \citep{fu2023monarch}, which fits an 8K--32K-token encoder into 80M parameters and is the strongest reference on LoCo when trained in-distribution \citep{saad2024benchmarking}. These approaches change the attention machinery; REIGN keeps a standard Transformer over chunk embeddings and pushes the long-context cost outside the training loop.

The closest settings to ours are therefore SMITH-style Siamese document matching and PARADE-style passage aggregation, but with a frozen, cached chunk encoder, at first-stage retrieval scale, and on documents averaging 5{,}065 words versus ToBERT's largest 1{,}788-word corpus \citep{cieri2004fisher}.

\subsection{Efficient and Lightweight Training}

Parameter-efficient tuning methods (Adapters \citep{houlsby2019parameter}, LoRA \citep{hu2022lora}, QLoRA \citep{dettmers2023qlora}, DoRA \citep{liu2024dora}) cut the trainable parameter count when adapting an LLM, but each training step still does a full token-level forward pass over the document. REIGN cuts cost on the other axis. Caching the GN forward passes removes the expensive computation from the training loop entirely, so the trainable encoder only ever sees a short, fixed-length sequence of chunk embeddings.

\section{Dataset}\label{sec:dataset}

\subsection{Challenges with Existing Datasets}

Existing embedding benchmarks target short-text similarity at the
sentence or paragraph level; document-level extensions rarely exceed
a few thousand tokens and lack the structural supervision long-context
models such as REIGN require. Although large long-form corpora
(Wikipedia \citep{wikidump}, Common Crawl
\citep{commoncrawl_cc_main_2024_10}, BookCorpus
\citep{zhu2015aligning}, Project Gutenberg
\citep{faysse_project_gutenberg_2023}) provide text, they do not
contain the semantically aligned long-document pairing required for
contrastive learning, and often exhibit noisy formatting or multi-topic
mixing.

\begin{figure*}[t]
  \centering
  \includegraphics[width=\linewidth]{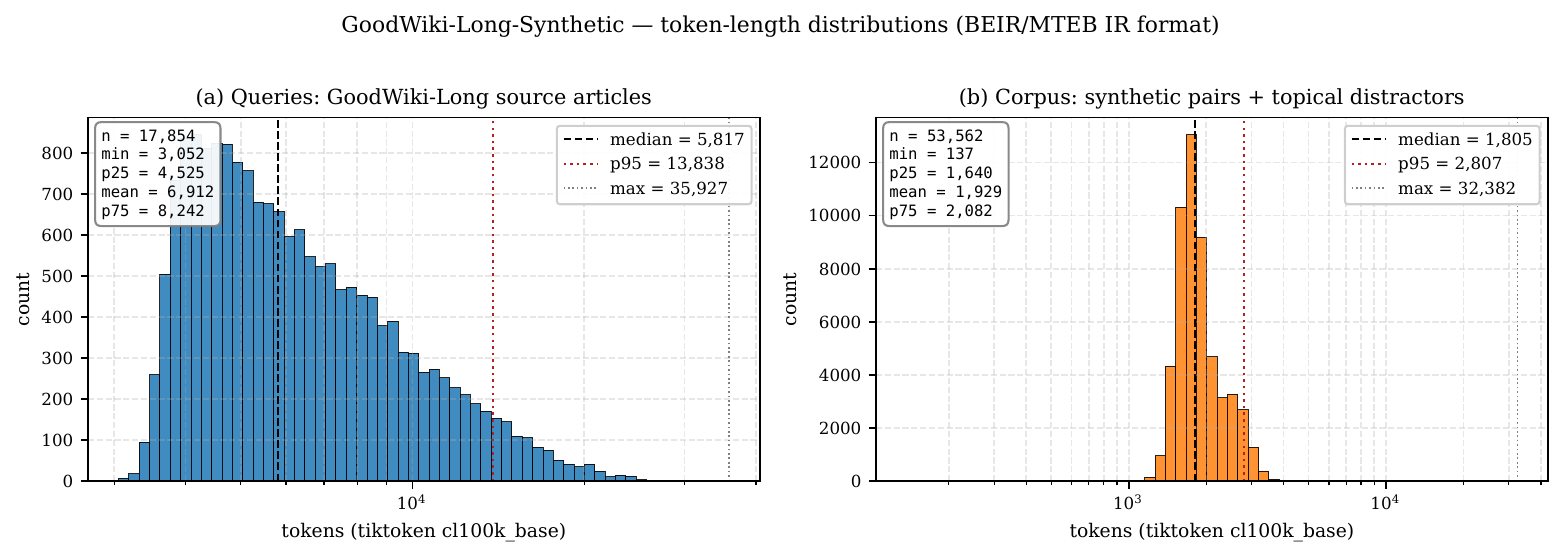}
  \caption{\textbf{Token-length distributions of
  \textsc{GoodWiki-Long-Synthetic}} (\texttt{tiktoken cl100k\_base}).
  \textbf{(a) Queries} ($n{=}17{,}854$) are GoodWiki-Long source
  articles, long-tailed (median $\sim$5.8K, p95 $\sim$13.8K).
  \textbf{(b) Corpus} ($n{=}53{,}562$) holds LLM rephrasals and topical
  distractors near a 1.8K-token median. The asymmetry puts the
  long-context burden on the query side, the regime REIGN's
  cross-chunk encoder targets.}
  \label{fig:goodwiki_long_hist}
\end{figure*}

\subsection{Synthetic Long-Document Dataset}

We construct a long-document retrieval dataset by rephrasing long-form
Wikipedia articles with GPT-4o-mini \citep{openai2024gpt4ocard}, so each
document and its rephrasal form a semantically aligned positive pair that
supports contrastive learning over long contexts. We build on GoodWiki
\citep{GoodWiki}, a cleaned Wikipedia release in structured markdown,
and filter for articles exceeding 16{,}000 characters ($\approx$4{,}000
tokens, following \citealp{tokencount2024openai}); we refer to this
subset as \textsc{GoodWiki-Long}.

To introduce harder negative cases, we augment positives with
semantically overlapping distractors inspired by HotpotQA
\citep{yang2018hotpotqa}, yielding a graded relevance structure.
During training, fully relevant pairs
enter the objective as positives and distractors as partial matches at
a reduced target weight (\Cref{ssec:objective}); at evaluation time we
use the graded labels to assess ranking robustness.

\paragraph{Open-access contribution.}
We publicly release \textsc{GoodWiki-Long-Synthetic} in the canonical
BEIR/MTEB tri-config
layout,\footnote{\url{https://huggingface.co/datasets/devrim/goodwiki_long_synthetic_ir}}
together with the full training and evaluation
code.\footnote{\url{https://github.com/devrimcavusoglu/reign}}
We refer to this benchmark as \textsc{GoodWiki-Long} in the
experiments. Pool sizes, qrels splits, and
document-length summary statistics are in
\Cref{tab:goodwiki_long_synthetic_ir}; full token-length distributions
appear in \Cref{fig:goodwiki_long_hist}
(\texttt{tiktoken cl100k\_base}, \citealp{tiktoken2023}).

\begin{table}[t]
  \centering
  \caption{Statistics for \textsc{GoodWiki-Long-Synthetic} in the
  BEIR/MTEB IR layout. Per-split rows are query-disjoint qrels, and
  length stats use \texttt{tiktoken cl100k\_base}.\\
  score~$=$~2 marks the LLM-rephrased positive, score~$=$~1 a topical
  distractor.}
  \label{tab:goodwiki_long_synthetic_ir}
  \scriptsize
  \setlength{\tabcolsep}{3pt}
  \begin{tabular}{lrrrr}
    \toprule
    \multicolumn{5}{l}{\emph{Pools}} \\
    \multicolumn{4}{l}{Queries (long-form Wikipedia)} & 17{,}854 \\
    \multicolumn{4}{l}{Corpus (positives $+$ distractors)} & 53{,}562 \\
    \multicolumn{4}{l}{\quad of which positives ($\text{score}{=}2$, one per query)} & 17{,}854 \\
    \multicolumn{4}{l}{\quad of which distractors ($\text{score}{=}1$, $\approx2$ per query)} & 35{,}708 \\
    \midrule
    \multicolumn{5}{l}{\emph{Qrels splits (query-disjoint)}} \\
    Split & Queries & Qrels & score=2 & score=1 \\
    train &  10{,}000 & 30{,}000 & 10{,}000 & 20{,}000 \\
    val   &   2{,}000 &  6{,}000 &  2{,}000 &  4{,}000 \\
    test  &   5{,}854 & 17{,}562 &  5{,}854 & 11{,}708 \\
    \midrule
    \multicolumn{5}{l}{\emph{Document length, tiktoken cl100k\_base (median / p95 / max)}} \\
    \multicolumn{4}{l}{Query side} & 5{,}817 / 13{,}838 / 35{,}927 \\
    \multicolumn{4}{l}{Corpus side} & 1{,}805 / \phantom{0}2{,}807 / 32{,}382 \\
    \bottomrule
  \end{tabular}
\end{table}

\section{Method}\label{sec:method}

REIGN is trained under a SimCLR-style contrastive
scheme~\citep{chen2020simple}. Each document is transformed into
two semantically equivalent views via rephrasing, forming positive
pairs $(x, x^+)$; injected distractors act as graded partial matches
and other batch samples as negatives. Each view is encoded into
chunk-level embeddings by a frozen Guidance Network (GN) and refined
by the REIGN encoder under a three-way cosine embedding objective
(\Cref{ssec:objective}). \Cref{fig:reign_simclr} diagrams the full
pipeline.

\begin{figure}[t]
    \centering
    \includegraphics[width=0.75\linewidth]{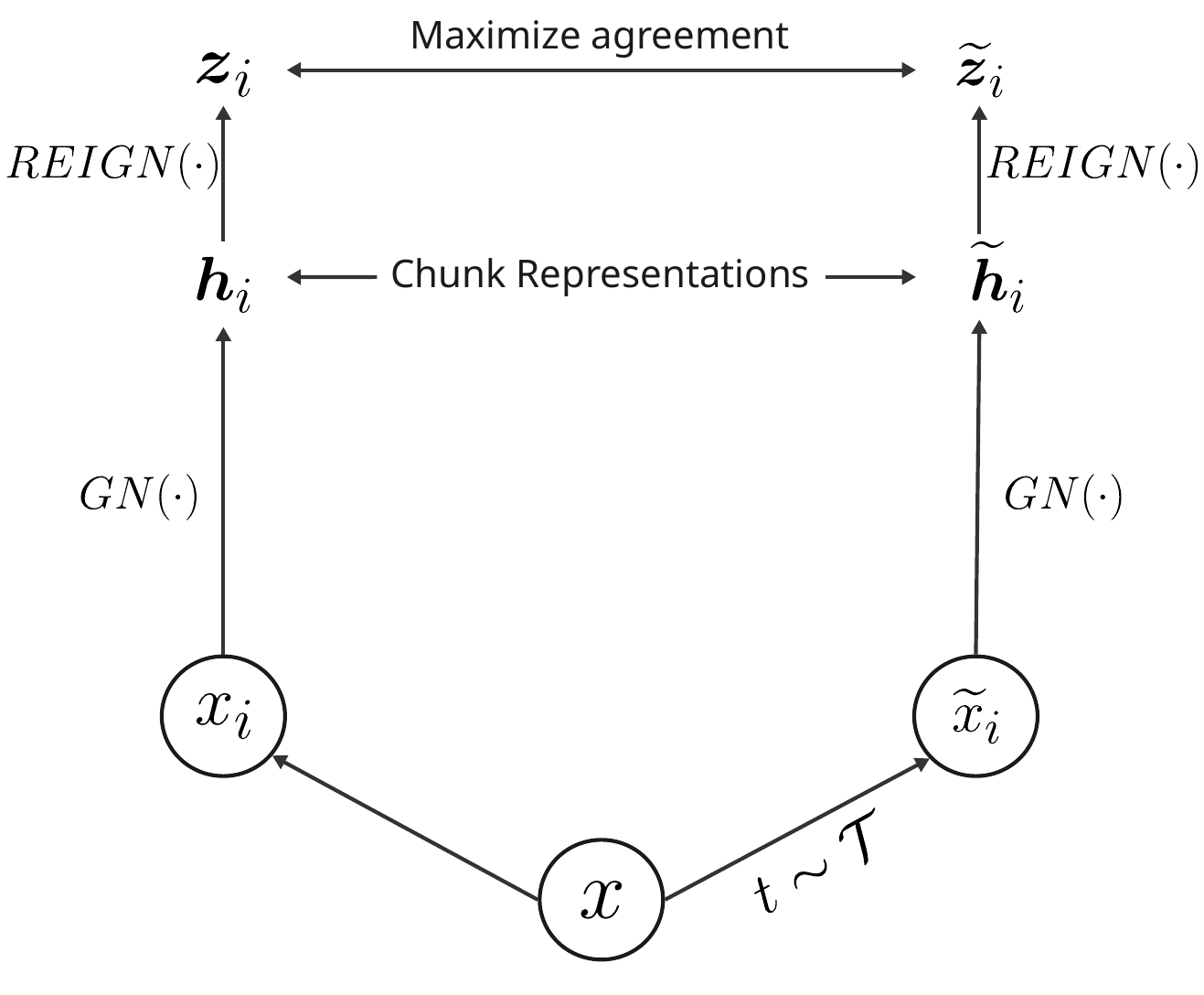}
    \caption{\textbf{Training pipeline of REIGN.}
    Each document $x$ is chunk-encoded by the frozen GN and refined by
    the REIGN encoder; semantic rephrasing ($t \sim \mathcal{T}$)
    generates positive pairs, and in-batch negatives enforce
    separation. Only the positive-pair path is drawn; injected
    distractors enter the objective as graded partials
    (\Cref{ssec:objective}).}
    \label{fig:reign_simclr}
\end{figure}

\subsection{Framework Overview}\label{ssec:framework}
REIGN treats a pretrained text encoder as a Guidance Network (GN)
that maps raw subword tokens to fixed-dimensional chunk embeddings;
the REIGN encoder then operates entirely at this coarser embedding
level, trading token detail for tractable context length
(\Cref{fig:strats}).

A textual input of up to $M$ tokens is segmented by a sliding window
of size $K$ (matching the GN's context window, typically $K=512$)
advanced with stride $S \le K$; consecutive windows overlap by $K-S$
tokens, with $S=K$ giving non-overlapping chunking. Each window $C_i$
is encoded as $E_i = \mathrm{GN}(C_i)$, yielding a sequence
$\{E_1, \dots, E_N\}$ with
\begin{equation}
N \;=\; \max\!\left(1,\; \left\lceil \tfrac{M - K}{S} \right\rceil + 1\right),
\end{equation}
which reduces to $\lceil M/K \rceil$ when $S=K$. GN outputs are
deterministic for a fixed GN/chunk-size/stride/input, so we
content-hash and cache them to disk (HDF5); during training only the
REIGN encoder runs, over a fixed-length embedding sequence. This
compresses the per-document cost from $O(M^2 d)$ at the token level
to $O(N^2 d)$ at the chunk level, so lightweight models can encode
inputs of $10^5{+}$ tokens; \Cref{apx:theory} formalises what the
move to chunk embeddings discards and what training can recover.

The REIGN encoder is a small Transformer applied to the
chunk-embedding sequence $\{E_1,\dots,E_N\}$ as a
permutation-equivariant set function with no positional encoding;
average pooling across chunks then makes the document representation
invariant to chunk order. Prior work and our own ablation support
this choice: chunk-level positional encoding gave ToBERT only marginal gains
\citep{pappagari2019hierarchical}, and in our own ablation learned
absolute positions are near parity on all three benchmarks while
sinusoidal encodings are worse
(\Cref{apx:pe_ablation}).\footnote{The ablation arms share a
controlled protocol (InfoNCE, batch 48, 20 epochs) distinct from the
released recipe; see \Cref{apx:pe_ablation}.} We therefore keep the
encoder free of positional parameters. A linear projection adapts the GN
output dimension to the encoder hidden width when they differ (e.g.\
$1024{\to}768$ for \textsc{GTE-large}). We sweep depth/width over
\textsc{tiny-l1} (0.56M), \textsc{small-l2} (3.85M), \textsc{base-l3}
(22.45M), \textsc{large-l4} (52.49M); see \Cref{apx:ablation_indist}
for the full sweep, and \Cref{apx:goodwiki_train} for training
hyperparameters.

\paragraph{Operating regime.}
An input that fits a single GN window ($M \le K$, i.e.\ $N{=}1$) is
\emph{short}. Its chunk sequence is one embedding, so the cross-chunk
encoder has nothing to aggregate. REIGN's operating regime therefore
starts at $N \ge 2$; at $K{=}512$ this means inputs beyond one
512-token chunk. Single-chunk inputs go to the bare GN, since on them
REIGN trails the GN by 5--7 nDCG@10 (\Cref{apx:mteb}) and adds an
encoder pass. In deployment
REIGN serves the long, multi-chunk side of a corpus alongside a
standard short-passage retriever.

\begin{figure}[!t]
\centering
\includegraphics[width=.85\linewidth]{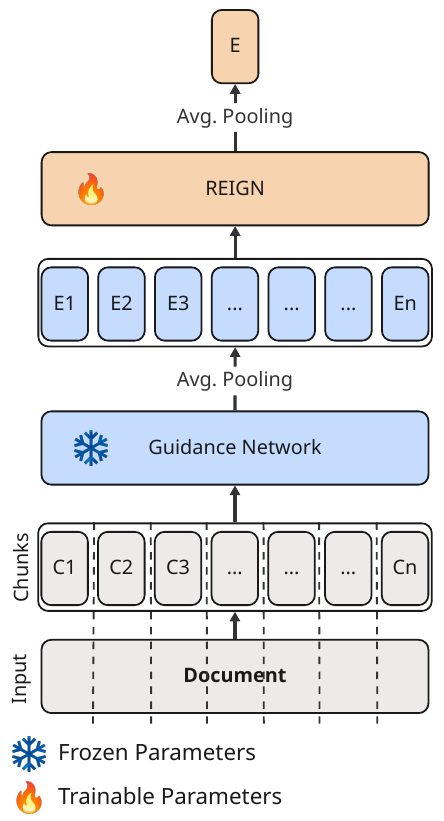}
\caption{\textbf{REIGN architecture.} A long document is chunked; the
frozen Guidance Network (GN) encodes each chunk and mean-pools its
token embeddings into one contextual chunk embedding
($E_1,\dots,E_n$); the trainable REIGN encoder aggregates across
chunks and average-pools into a single document representation~$E$.}
\label{fig:strats}
\end{figure}

\subsection{Training Objective}\label{ssec:objective}

REIGN is optimised with a three-way cosine embedding loss over
document pairs $(x, y)$ carrying graded targets
$s \in \{1, 0, -1\}$ (positive, partial, negative):
\begin{equation}
\mathcal{L}(x, y, s) =
\begin{cases}
1 - c & s = 1,\\
\lambda\,(1 - c) + (1 - \lambda)\,c_{+} & s = 0,\\
c_{+} & s = -1,
\end{cases}
\end{equation}
where $c = \cos(x, y)$, $c_{+} = \max(0, c)$, and the partial weight
is $\lambda = 0.5$. Each batch of 18 anchors pairs every anchor with
its rephrased positive, its two topical distractors as partials, and
17 in-batch negatives (360 pairs per step; full protocol in
\Cref{apx:goodwiki_train}). The choice of this objective over
InfoNCE~\citep{oord2018representation} is empirical. In a controlled
ablation this objective outperforms every InfoNCE variant we tried
(temperature, distractor handling, batch size, warm start, epochs) on
all three benchmarks (\Cref{apx:loss_ablation}). InfoNCE was also
brittle in this setting: at batch size 18 its validation score
degrades monotonically after the first epoch, and stable training
needs a much larger negative pool (\Cref{apx:loss_ablation}). The
pairwise objective has no such requirement. The \textsc{DAPFAM} fine-tuning
experiments (\Cref{ssec:dapfam}) instead use InfoNCE at temperature
0.07 over the dataset's provided negatives (\Cref{apx:dapfam}).

\subsection{Efficiency}\label{ssec:efficiency}

We compare REIGN's cost against representative long-context embedding
approaches analytically and empirically, under a fixed 8K-token,
512-token-chunk setting (\Cref{apx:efficiency},
\Cref{tab:flops_appendix_8k}). Token-level
Transformer models incur substantial overhead, linear in chunk count
for chunked inference and quadratic in sequence length for full-context
attention. REIGN's cost is instead dominated by a small embedding-level
encoder over a short sequence of chunk embeddings. With cached GN
embeddings, per-document training cost drops by nearly four orders of
magnitude relative to chunked Transformer fine-tuning. Even uncached,
REIGN matches standard chunked inference while enabling document-level
contrastive training that would otherwise be prohibitive.

End-to-end measurements on a single RTX~4090 agree with this
analysis (\Cref{tab:measured_efficiency} in \Cref{apx:efficiency}).
With cached GN embeddings REIGN answers
queries in 0.40--0.52\,ms versus 19.8--118.2\,ms with the GN run per
query ($49$--$229\times$), and the one-time cache build
(8.9--53.3\,ms per document) amortises almost immediately. Uncached,
REIGN runs at parity with its own chunked GN (ratio 0.86--0.97 over
identical GN work). Peak GPU memory spans 0.24--1.73\,GB across REIGN
configurations, versus 4.8--18.9\,GB for the native long-context
dense baselines (Jina-v3 at batch 4).

\section{Experiments}

We evaluate REIGN along three axes:
(i)~\textsc{GoodWiki-Long} (\Cref{ssec:exp_gw}, in-distribution
document-to-document retrieval, ``doc-to-doc'');
(ii)~\textsc{LoCo} (\Cref{ssec:loco}, OOD query$\to$passage, 12
subtasks across legal, scientific, and government text);
(iii)~\textsc{DAPFAM} (\Cref{ssec:dapfam}, OOD real-world patent
doc-to-doc).
Primary metrics are nDCG@10/@100 following
\citet{thakur2021beir, muennighoff2022mteb,saad2024benchmarking}.
One shared evaluation implementation scores every system (sparse,
dense, and REIGN) over identical qrels; on \textsc{GoodWiki-Long} the
graded labels enter nDCG with exponential gains
($2^{\mathrm{rel}}{-}1$), so every system receives the same partial
credit for distractors.
A short-context MTEB generalisation study is in \Cref{apx:mteb};
REIGN underperforms there by construction (see Limitations). The
encoder-capacity ablation that justifies our \textsc{base-l3}
paper-default is deferred to \Cref{apx:ablation_indist} and
\Cref{apx:ablation_ood_full}.

\begin{figure*}[t]
  \centering
  \includegraphics[width=\linewidth]{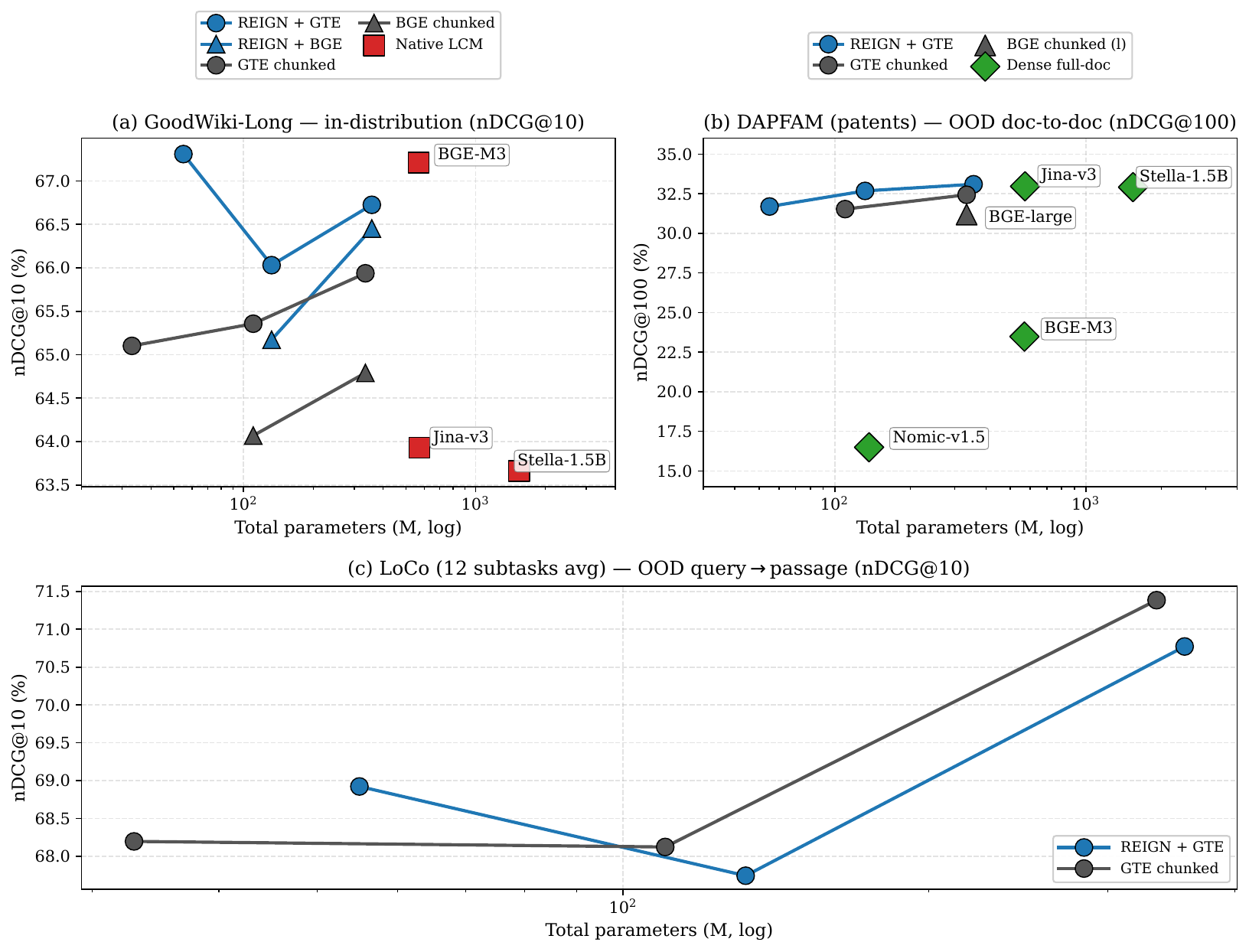}
  \caption{\textbf{Parameter efficiency of REIGN across three
  long-document retrieval regimes.} Total parameters (log-scale) vs.\
  primary nDCG; solid lines connect the three GN scales (small / base /
  large). \textbf{(a) GoodWiki-Long.} REIGN sits
  above the chunked baseline at every GN scale, and 55M REIGN+GTE-small
  beats the three plotted native long-context dense baselines
  (BGE-M3, Jina-v3, Stella-1.5B) at $10$--$28\times$ smaller scale;
  Nomic-Embed-v1.5 (11.05 nDCG@10) is omitted for scale.
  \textbf{(b) DAPFAM.} REIGN+GTE matches or exceeds chunked-GTE across
  scales; the strongest full-document dense baselines (Jina-v3,
  Stella-1.5B) reach comparable nDCG@100 with $1.6$--$4.3\times$ more
  parameters (\Cref{tab:dapfam_significance}). \textbf{(c) LoCo.} Bare
  GTE-chunked is at or above REIGN+GTE except at the small-GN
  end (see \Cref{ssec:loco}).}
  \label{fig:pareto_overview}
\end{figure*}

\subsection{Retrieval Accuracy on \textsc{GoodWiki-Long}}
\label{ssec:exp_gw}
\paragraph{Setup.}
\Cref{tab:gw_lcm_pareto} reports retrieval on the
\textsc{GoodWiki-Long} test split ($n_q{=}5{,}854$,
$n_{\text{corpus}}{=}53{,}562$) against sparse baselines (BM25,
TF-IDF), four native long-context dense baselines (BGE-M3,
\citealp{chen-etal-2024-m3}; Jina-Embeddings-v3, \citealp{sturua2024jina};
Nomic-Embed-v1.5, \citealp{nussbaum2024nomic}; Stella-en-1.5B-v5,
\citealp{stella2024}), the GN backbones (GTE, \citealp{li2023towards};
BGE, \citealp{xiao2024cpack}) used directly under truncated and
chunked mean-pool protocols, and a multi-GN REIGN sweep. Sparse retrievers
dominate raw nDCG@10 (BM25~76.82). \textsc{GoodWiki-Long}'s positives
are rephrased distillations of their queries, and the pairs retain
substantial lexical overlap, which favours bag-of-words.

\paragraph{Headline result.}
Within dense methods,
REIGN+GTE-small (55M; \Cref{fig:pareto_overview}a) tops nDCG@10,
edging BGE-M3 (568M)
by 0.10 with $10\times$ fewer parameters and beating Jina-v3 and
Stella-1.5B by 3.4--3.7~points, while exceeding the strongest bare-GN
chunked baseline (GTE-large, 335M) by 1.4. Chunked mean-pool buys
1.7--2.7~nDCG@10 over truncation across GNs; REIGN's cross-chunk
encoder adds a further 0.7--2.2 on top, with the largest lift at the
small-GN end.

\paragraph{GN scaling.}
Within REIGN, GoodWiki-Long performance is nearly saturated at
the small-GN end, since scaling the GN from 33M (\textsc{GTE-small})
to 335M (\textsc{GTE-large}) does not improve nDCG@10. REIGN is
therefore best paired with the smallest GN that crosses an
embedding-quality threshold.

\begin{table}[t]
\centering
\caption{\textbf{GoodWiki-Long test.}
$n_q{=}5{,}854$, $n_{\text{corpus}}{=}53{,}562$, top-$k{=}10$.
nDCG@10 in percent; $\Delta$ is the lift over the truncated
baseline of the same GN. REIGN uses the \textsc{base-l3}
encoder (22M) at the best-performing stride per GN.\\
\textbf{Bold} = best per column.\\
$^{\ddagger}$Nomic-Embed-v1.5's native training context is 2K tokens,
and \citet{nussbaum2024nomic} mark its 8K RoPE extension as
experimental; it collapses on documents whose median length already
exceeds 5K.}
\label{tab:gw_lcm_pareto}
\scriptsize
\setlength{\tabcolsep}{4pt}
\begin{tabular}{lrrr}
\toprule
\textbf{System} & \textbf{Params} & \textbf{nDCG@10} & \textbf{$\Delta$} \\
\midrule
\multicolumn{4}{l}{\emph{Sparse lexical}}\\
BM25                                & --            & \textbf{76.82} & --              \\
TF-IDF                              & --            & 74.19          & --              \\
\midrule
\multicolumn{4}{l}{\emph{Native long-context dense (truncate to 8K)}}\\
BGE-M3                              & 568M          & 67.21          & --              \\
Jina-Embeddings-v3                  & 572M          & 63.93          & --              \\
Stella-en-1.5B-v5                   & 1.54B         & 63.66          & --              \\
Nomic-Embed-v1.5\,$^{\ddagger}$     & 137M          & 11.05          & --              \\
\midrule
\multicolumn{4}{l}{\emph{Bare GN (truncated, 512)}}\\
GTE-small                           & 33M           & 62.37          & --              \\
GTE-base                            & 110M          & 63.48          & --              \\
GTE-large                           & 335M          & 63.58          & --              \\
BGE-base                            & 110M          & 61.95          & --              \\
BGE-large                           & 335M          & 63.12          & --              \\
\midrule
\multicolumn{4}{l}{\emph{Bare GN (chunked mean-pool, 512)}}\\
GTE-small                           & 33M           & 65.10          & $+2.73$         \\
GTE-base                            & 110M          & 65.36          & $+1.88$         \\
GTE-large                           & 335M          & 65.94          & $+2.36$         \\
BGE-base                            & 110M          & 64.07          & $+2.12$         \\
BGE-large                           & 335M          & 64.79          & $+1.67$         \\
\midrule
\multicolumn{4}{l}{\emph{REIGN (this work)}}\\
REIGN + GTE-small                   & 55M           & 67.31          & $\mathbf{+4.94}$ \\
REIGN + BGE-base                    & 132M          & 65.17          & $+3.22$         \\
REIGN + GTE-base                    & 132M          & 66.03          & $+2.55$         \\
REIGN + BGE-large                   & 357M          & 66.45          & $+3.33$         \\
REIGN + GTE-large                   & 357M          & 66.73          & $+3.15$         \\
\bottomrule
\end{tabular}
\end{table}

\subsection{Out-of-Distribution Long-Context Retrieval: LoCo}
\label{ssec:loco}
\paragraph{Setup.}
\Cref{tab:loco} reports REIGN zero-shot on LoCo
\citep{saad2024benchmarking}, an OOD suite of 12 long-context
retrieval subtasks (legal opinions, scientific QA, government
reports, summarisation, code search), each with its own queries,
corpus, and qrels. Following the benchmark convention we report
per-subtask nDCG@10, and compare systems on the 12-subtask
macro-average (\Cref{fig:pareto_overview}c). The REIGN encoder is the
\textsc{base-l3} checkpoint trained solely on \textsc{GoodWiki-Long},
paired with \textsc{GTE-small}, \textsc{GTE-base}, and
\textsc{GTE-large} at strides 384 and 512; no LoCo-specific training.

\begin{table*}[t]
\centering
\caption{\textbf{LoCoV1 per-subtask nDCG@10.}
Baselines reproduced from \citet{saad2024benchmarking}, Table~13.
``Ctx'' is the encoder's native context window.\\
$\infty^{\dagger}$ marks REIGN (chunked at 512, effective context
bounded only by GPU memory).\\
Above the midrule: zero-shot retrievers, \textbf{bold} = best per
column. M2-BERT below it was trained on data including the
LoCo training splits \citep{saad2024benchmarking}, so it is shown for
reference rather than as a like-for-like comparison; its large
legal-subtask margins come from in-distribution supervision.\\
Abbreviations: 2wMQA, CL\_H/T, GovR, LCR, MFQA, PR,
QasA/T, QMS, SO, SSFD = 2wikimqa, courtlistener\_HTML/Plain,
gov\_report, legal\_case\_reports, multifieldqa, passage\_retrieval,
qasper\_abstract/title, qmsum, stackoverflow, summ\_screen\_fd.}
\label{tab:loco}
\scriptsize
\setlength{\tabcolsep}{3pt}
\begin{tabular}{lrrrrrrrrrrrrrrr}
\toprule
\textbf{System} & \textbf{Params} & \textbf{Ctx} & \textbf{2wMQA} & \textbf{CL\_H} & \textbf{CL\_T} & \textbf{GovR} & \textbf{LCR} & \textbf{MFQA} & \textbf{PR} & \textbf{QasA} & \textbf{QasT} & \textbf{QMS} & \textbf{SO} & \textbf{SSFD} & \textbf{Avg} \\
\midrule
BGE-Large                  & 335M  & 512                   & 69.4 & 10.0 & 10.1 & 92.6 & 18.9 & 89.3 & 20.1 & 94.5 & 87.1 & 46.0 & 74.1 & 65.8 & 56.5 \\
OpenAI Ada-002             & --    & 8K                    & 78.9 & 16.3 & 16.8 & 97.1 & 28.2 & 90.1 & 31.2 & 98.9 & 93.8 & 48.8 & 72.3 & 86.2 & 63.2 \\
Jina-Embeddings-v2         & 137M  & 8K                    & 81.6 & 27.0 & 26.1 & 98.6 & 30.7 & 86.4 & 60.7 & 99.4 & 95.1 & 40.5 & 69.0 & 93.3 & 67.2 \\
E5-Mistral                 & 7.11B & 4K                    & \textbf{88.3} & \textbf{33.9} & \textbf{34.6} & 98.3 & \textbf{49.5} & \textbf{93.5} & 35.3 & \textbf{99.8} & \textbf{98.4} & 46.8 & \textbf{82.7} & \textbf{95.9} & \textbf{71.4} \\
\addlinespace[1pt]
GTE-small (chunked)        & 33M   & 512                   & 78.67 & 28.99 & 26.68 & 98.67          & 33.83 & 88.41 & 49.88          & 99.67 & 93.76 & 64.48 & 61.70 & 93.61 & 68.20 \\
GTE-base  (chunked)        & 110M  & 512                   & 82.32 & 27.91 & 26.53 & \textbf{98.76} & 31.81 & 91.14 & 50.92          & 99.32 & 94.29 & 63.36 & 62.06 & 89.07 & 68.12 \\
GTE-large (chunked)        & 335M  & 512                   & 85.78 & 31.16 & 29.28 & 98.45          & 39.75 & 92.06 & 62.76          & 99.70 & 96.73 & 63.12 & 63.56 & 94.25 & 71.38 \\
REIGN + GTE-small @s384    & 55M   & $\infty^{\dagger}$    & 79.65 & 27.23 & 26.41 & 98.55          & 34.26 & 88.19 & 63.18          & 98.49 & 90.19 & 63.08 & 63.46 & 93.75 & 68.87 \\
REIGN + GTE-small @s512    & 55M   & $\infty^{\dagger}$    & 86.15 & 26.15 & 24.89 & 98.67          & 32.69 & 89.45 & 60.53          & 98.69 & 90.11 & 63.37 & 62.69 & 93.68 & 68.92 \\
REIGN + GTE-base  @s384    & 132M  & $\infty^{\dagger}$    & 81.68 & 24.73 & 24.41 & 98.41          & 29.12 & 87.56 & 60.19          & 98.72 & 90.72 & 60.48 & 65.45 & 91.46 & 67.74 \\
REIGN + GTE-base  @s512    & 132M  & $\infty^{\dagger}$    & 80.96 & 24.86 & 23.67 & 98.53          & 30.07 & 88.93 & 56.39          & 98.97 & 91.22 & 60.98 & 65.25 & 92.79 & 67.72 \\
REIGN + GTE-large @s384    & 357M  & $\infty^{\dagger}$    & 83.90 & 27.69 & 26.70 & 98.62          & 40.11 & 88.36 & \textbf{64.76} & 99.47 & 94.21 & \textbf{64.59} & 67.37 & 93.49 & 70.77 \\
REIGN + GTE-large @s512    & 357M  & $\infty^{\dagger}$    & 84.94 & 27.56 & 26.14 & 98.74          & 39.98 & 89.96 & 61.18          & 99.56 & 94.93 & 63.00 & 66.81 & 93.09 & 70.49 \\
\midrule
\multicolumn{16}{l}{\emph{Trained on LoCo splits; not a zero-shot comparison.}}\\
M2-BERT-8K                 & 80M   & 8K                    & 87.7 & 94.9 & 95.4 & 96.5 & 58.6 & 97.1 & 82.0 & 98.9 & 93.1 & 69.6 & 84.8 & 84.6 & 86.9 \\
M2-BERT-32K                & 80M   & 32K                   & 92.1 & 96.8 & 97.0 & 98.7 & 67.1 & 98.5 & 89.3 & 98.3 & 97.9 & 70.4 & 88.9 & 98.0 & 91.1 \\
\bottomrule
\end{tabular}
\end{table*}

\paragraph{Headline numbers.}
The bare-GN chunked baseline \textsc{GTE-small (chunked)} already
reaches 68.20~macro-average nDCG@10, outperforming BGE-Large (56.5),
OpenAI Ada-002 (63.2), and Jina-Embeddings-v2 (67.2). Adding the
REIGN encoder on top (55M total) lifts the macro-average to 68.92
($+0.72$ over its GN). With a larger GN,
\textsc{REIGN+GTE-large} (357M total) reaches 70.77, within 0.65 of
the $20\times$-larger E5-Mistral \citep{wang-etal-2024-improving-text}
(71.4), and achieves the best zero-shot
scores on \texttt{passage\_retrieval} (64.76) and \texttt{qmsum}
(64.59).

\paragraph{Where REIGN helps over its GN.}
REIGN's added value concentrates at the small-GN end. Relative to the
same-GN chunked baseline it is $+0.72$~macro on \textsc{GTE-small},
$-0.38$ on \textsc{GTE-base}, and
$-0.61$ on \textsc{GTE-large};
bare \textsc{GTE-large (chunked)} is the highest LoCo
macro-average we observe across any configuration tested
(71.38, $0.02$ below E5-Mistral at $20\times$ smaller scale).
The per-subtask pattern explains the average. REIGN gives large lifts
where evidence must be aggregated across chunks
(\texttt{passage\_retrieval} $+13.30$ at \textsc{GTE-small},
\texttt{2wikimqa} $+7.48$), but a stronger GN's chunked mean-pool
already saturates the query$\rightarrow$passage subtasks that dominate
LoCo (at the bare GN, \texttt{gov\_report} and \texttt{qasper\_abstract}
are ${\geq}98$; \texttt{qasper\_title} and \texttt{summ\_screen\_fd}
are ${\geq}89$). REIGN and its GN alike are weak on the lexically
driven legal subtasks (\texttt{courtlistener\_*},
\texttt{legal\_case\_reports}), a term-overlap regime the
embedding-level paradigm cannot exploit (see Limitations).

\paragraph{Non-monotone GN scaling.}
The non-monotone GN-scale pattern (REIGN+GTE-base worse than both small
and large) also appears in-distribution (\Cref{tab:ablation_indist}),
so it is not LoCo-specific; we tentatively attribute it to a mismatch
between \textsc{GTE-base}'s chunk-embedding geometry and the fixed 22M
REIGN-encoder capacity, and leave a co-varied sweep to future work.

\subsection{Real-World Validation: Patent Retrieval on DAPFAM}
\label{ssec:dapfam}
To test whether REIGN's gains on synthetic long-form Wikipedia carry
over to a real specialised domain, we evaluate on \textsc{DAPFAM}
\citep{ayaou2025dapfam}, a domain-adaptive patent retrieval benchmark
of USPTO patent families. Documents are substantially longer than
those in typical retrieval corpora (median $\sim$7--12K whitespace tokens; tail
to several hundred thousand), squarely in REIGN's target regime.
We use the \textsc{FullText} view throughout.

\paragraph{Setup.}
\textsc{DAPFAM} ships eval-only, so we build a query-disjoint
70/15/15 stratified train/val/test partition (seed 42), holding out
$n{=}187$ test queries. We retain the authors' IN-/OUT-IPC distinction
\emph{within} the held-out set, where \textsc{test\_in} ($n{=}180$) and
\textsc{test\_out} ($n{=}138$) are positives-only subsets of \textsc{test},
so the cross-domain breakdown is reported on identical queries across
systems. Retrieval is against the full 45{,}336-document FullText
corpus with self-matches removed; the authoritative metric is
nDCG@100. Fine-tuning uses the dataset's random \texttt{score}{=}0
negatives, with the objective's partial branch disabled (further
details in \Cref{apx:dapfam}).

\begin{table}[t]
\centering
\caption{\textbf{Patent retrieval on \textsc{DAPFAM}.} nDCG@100 on
\textsc{test} ($n{=}187$), IN-IPC subset \textsc{test\_in} ($n{=}180$),
and cross-IPC subset \textsc{test\_out} ($n{=}138$); 45{,}336-doc
FullText corpus. $\Delta$ is the lift on \textsc{test} over the
same-GN bare chunked baseline (in/out columns not evaluated for sparse
and full-document baselines). REIGN rows use the
\textsc{GoodWiki-Long}-trained zero-shot \textsc{base-l3} checkpoint at
the best-performing stride per GN.\\
\textbf{Bold} = best mean per
column; on \textsc{test} the margins between REIGN and the strongest
dense baselines are not statistically significant
(\Cref{tab:dapfam_significance}).\\
$^{\ddagger}$Nomic-Embed-v1.5's native training context is 2K tokens,
and \citet{nussbaum2024nomic} mark its 8K RoPE extension as
experimental; it collapses on documents whose median length already
exceeds the native window.}
\label{tab:dapfam_headline}
\scriptsize
\setlength{\tabcolsep}{3pt}
\begin{tabular}{lrrrrr}
\toprule
\textbf{System} & \textbf{Params} & \textbf{test} & \textbf{test\_in} & \textbf{test\_out} & \textbf{$\Delta$} \\
\midrule
\multicolumn{6}{l}{\emph{Sparse lexical}}\\
BM25                                & --           & 25.62          & --             & --             & --              \\
TF-IDF                              & --           & 30.06          & --             & --             & --              \\
\midrule
\multicolumn{6}{l}{\emph{Native long-context dense (truncate to 8K)}}\\
BGE-M3                              & 568M         & 23.48          & --             & --             & --              \\
Jina-Embeddings-v3                  & 572M         & 32.97          & --             & --             & --              \\
Nomic-Embed-v1.5\,$^{\ddagger}$     & 137M         & 16.49          & --             & --             & --              \\
Stella-en-1.5B-v5                   & 1.54B        & 32.91          & --             & --             & --              \\
\midrule
\multicolumn{6}{l}{\emph{Bare GN (chunked mean-pool, 512)}}\\
GTE-small                           & 33M          & 29.95          & 34.92          & 4.63           & --              \\
GTE-base                            & 110M         & 31.53          & 36.68          & 5.14           & --              \\
GTE-large                           & 335M         & 32.43          & 37.50          & 5.22           & --              \\
BGE-large                           & 335M         & 31.19          & --             & --             & --              \\
\midrule
\multicolumn{6}{l}{\emph{REIGN (this work)}}\\
REIGN + GTE-small                   & 55M          & 31.69          & 36.97          & 4.95           & $\mathbf{+1.74}$ \\
REIGN + GTE-base                    & 132M         & 32.72          & 37.65          & \textbf{5.55}  & $+1.19$         \\
REIGN + GTE-large                   & 357M         & \textbf{33.10} & \textbf{38.15} & 5.02           & $+0.67$         \\
\bottomrule
\end{tabular}
\end{table}

\paragraph{Headline result.}
\Cref{tab:dapfam_headline} compares REIGN against sparse (BM25, TF-IDF),
native long-context dense (BGE-M3, Jina-Embeddings-v3, Nomic-Embed-v1.5,
Stella-en-1.5B-v5), and chunked dense (BGE-large, GTE-small/base/large)
baselines. REIGN trained \emph{only} on \textsc{GoodWiki-Long}, with no
patent-specific adaptation, has the best mean, 33.10~nDCG@100
(REIGN+\textsc{GTE-large}, 357M total; \Cref{fig:pareto_overview}b),
and is statistically
indistinguishable from the strongest single-stage dense baselines
Jina-v3 (572M) and Stella-1.5B (1.54B) with $1.6$--$4.3\times$ fewer
parameters (\Cref{tab:dapfam_significance}); the lift over
\textsc{GTE-large} chunked alone is $+0.67$. nDCG@100 grows
monotonically with GN size ($31.69 \to 32.72 \to 33.10$, best stride
per GN).

\begin{table}[t]
\centering
\caption{\textbf{Paired significance on \textsc{DAPFAM} test.}
REIGN+\textsc{GTE-large} @s512 (33.10) vs.\ each baseline
on per-query nDCG@100 ($n{=}187$). Paired bootstrap 95\% CI
($B{=}10{,}000$) and paired randomisation $p$ ($B{=}10{,}000$),
Holm-adjusted across the five comparisons ($p_{\text{H}}$). Deltas are
computed on unrounded per-query means from the significance re-runs,
whose provenance is given in
\Cref{apx:dapfam}.\\
W/L/T = per-query wins/losses/ties.}
\label{tab:dapfam_significance}
\scriptsize
\setlength{\tabcolsep}{2.1pt}
\begin{tabular}{lrcrrc}
\toprule
\textbf{Baseline} & \textbf{$\Delta$} & \textbf{95\% CI} & \textbf{$p$} & \textbf{$p_{\text{H}}$} & \textbf{W/L/T} \\
\midrule
Jina-v3           & $+0.12$ & $[-1.13,+1.36]$ & .857    & 1.000   & 81/95/11  \\
Stella-1.5B       & $+0.09$ & $[-1.71,+1.98]$ & .921    & 1.000   & 80/95/12  \\
GTE-large chunked & $+0.67$ & $[+0.01,+1.34]$ & .051    & .154    & 103/73/11 \\
BM25              & $+7.48$ & $[+6.03,+9.03]$ & $<$.001 & $<$.001 & 136/43/8  \\
TF-IDF            & $+3.03$ & $[+1.50,+4.60]$ & $<$.001 & $<$.001 & 103/74/10 \\
\bottomrule
\end{tabular}
\end{table}

\paragraph{Statistical significance.}
We test the \textsc{test}-split margins with a paired bootstrap
($B{=}10{,}000$, seed 42) for confidence intervals and a paired
randomisation test ($B{=}10{,}000$) for $p$-values, Holm-adjusted
across the five comparisons (\Cref{tab:dapfam_significance}). All
three dense margins are non-significant. REIGN is at statistical
parity with Jina-v3 and Stella-1.5B while using $1.6$--$4.3\times$
fewer parameters, and against both it wins on fewer queries than it loses
(81/95 and 80/95); its higher mean comes from larger gains on the
queries it wins. Against its own chunked GN the margin of $+0.67$
is a boundary case, since the bootstrap CI marginally excludes zero
($[+0.01, +1.34]$) while the randomisation test does not reach
significance ($p{=}.051$; Holm-adjusted $.154$). REIGN is also the
\emph{larger} system in that comparison (357M vs.\ 335M), so the
parameter-efficiency framing does not apply there. REIGN beats both
sparse baselines significantly; sparse retrieval leads on the
synthetic \textsc{GoodWiki-Long} (\Cref{ssec:exp_gw}) but not on
these human-written patents.

\paragraph{Cross-domain decomposition.}
The \textsc{test\_in}/\textsc{test\_out} columns split the held-out
set into IN-IPC and cross-IPC positives. Absolute nDCG is dominated by
IN-IPC ($\sim$38\%); the cross-IPC partition plateaus in the
$4.6$--$5.6$ range across the full GTE scale axis (chunked
small/base/large) and every REIGN variant. This plateau is an
intrinsic property of the OUT-split (random non-cited families
provide too weak a training contrast against a topically dense 45K
corpus, a limitation noted by the dataset authors), not a model-side
deficit. On IN-IPC, REIGN's lift over the matched bare-GN chunked
baseline narrows with GN scale ($+2.05$ at \textsc{GTE-small}, $+0.97$
at \textsc{GTE-base}, $+0.65$ at \textsc{GTE-large}), mirroring the
in-distribution \textsc{GoodWiki-Long} pattern.

\paragraph{Fine-tuning does not exceed zero-shot.}
To rule out under-tuning, we swept fine-tuning configurations over
warm- vs.\ cold-start, backbone scale, stride, and a learning-rate
$\times$ weight-decay grid (\Cref{apx:dapfam}). The best run, a
regularised warm-start at lr~$10^{-6}$ and wd~$10^{-2}$, reaches
32.73~nDCG@100, statistical parity with the zero-shot checkpoint it
started from ($32.68 \to 32.73$). Naive fine-tuning at lr~$10^{-5}$
degrades by 0.4--1.5 points via catastrophic forgetting (warm) or
overfit (cold); cold-start loses an additional 0.43 to warm-start at
matched hyperparameters (32.30 vs.\ 32.73), confirming that the
\textsc{GoodWiki-Long}-trained encoder carries useful inductive bias
outside its training domain.

\section{Conclusion}
\label{sec:conclusion}

We presented \textbf{REIGN}, a framework for scalable document-level
representation learning that operates directly on precomputed chunk
embeddings from a frozen Guidance Network and trains a lightweight
cross-chunk encoder under a contrastive objective. The design (i)
scales to inputs of $10^5{+}$ tokens without architectural
modifications,
(ii) reduces per-document training and inference cost by orders of
magnitude through cached GN embeddings (\Cref{ssec:efficiency}), and
(iii) transfers zero-shot from synthetic long-form Wikipedia to a
specialised long-document domain (IN-IPC patent retrieval). Training
runs on \textsc{GoodWiki-Long-Synthetic}, a synthetic long-document
retrieval benchmark we release, and \Cref{apx:theory} separates the
information the frozen GN discards from what training can recover.

The smallest REIGN configuration (55M, GTE-small) outperforms the
LoCo-paper zero-shot baselines (BGE-Large, Jina-v2, Ada-002), the
first two at $2.5$--$6\times$ smaller scale, and REIGN+GTE-large reaches
70.77~macro nDCG@10 within 0.65~points of E5-Mistral with $20\times$
fewer parameters. On \textsc{DAPFAM} REIGN has the best mean overall
(33.10~nDCG@100) and on the IN-IPC partition, and is statistically
indistinguishable from the strongest single-stage dense baselines
Jina-v3 (572M) and Stella-1.5B (1.54B) with $1.6$--$4.3\times$ fewer
total parameters (\Cref{tab:dapfam_significance}); the cross-IPC
plateau ($\sim$5~nDCG@100
across every system) is a dataset property.

The LoCo picture is
task-conditional. REIGN delivers large lifts on multi-chunk-aggregation
subtasks, but the bare GTE-large chunked mean-pool is the Pareto-best
LoCo configuration overall. The in-distribution / OOD encoder-capacity
ablation (\Cref{apx:ablation_indist}, \Cref{apx:ablation_ood_full})
establishes \textsc{base-l3} as the principled default and surfaces a
methodological caveat. In-distribution sweet-spot claims must be
verified out-of-distribution.

Natural future
directions include adaptive or overlapping chunking, joint GN
fine-tuning, and hybrid composition with a sparse lexical channel.

\section*{Limitations}

\paragraph{Task-conditional value over the GN.}
REIGN's added value over its frozen Guidance Network scales with
how much cross-chunk reasoning the retrieval task actually requires.
On doc-to-doc retrieval the lift is consistent. REIGN+GTE-small (55M)
tops every dense baseline on \textsc{GoodWiki-Long}
(\Cref{ssec:exp_gw}), and REIGN+GTE-large (357M) has the best
\textsc{DAPFAM} mean against $1.6$--$4.3\times$-larger dense
baselines (\Cref{ssec:dapfam}); the margins are small
(0.10~nDCG@10 over BGE-M3 in-distribution) and on \textsc{DAPFAM} not
statistically significant (\Cref{tab:dapfam_significance}), so the
advantage is parameter efficiency rather than raw accuracy. On
query$\rightarrow$passage long-context retrieval (LoCo,
\Cref{ssec:loco}) the chunked mean-pool of a stronger GN already
saturates the per-subtask ceiling, and REIGN's encoder adds little
on \textsc{GTE-base} ($-0.38$ macro) and slightly hurts on
\textsc{GTE-large} ($-0.61$ macro). On short-passage retrieval
where the chunked architecture cannot be exercised at all
(MTEB \textsc{ArguAna}, \textsc{FiQA-2018}; \Cref{apx:mteb}),
performance drops by 5--7~nDCG@10. Across the three regimes,
REIGN's encoder is principally a doc-to-doc
cross-chunk-reasoning module, and deployment should target tasks whose
retrieval signal requires aggregating evidence across multiple
chunks of a document.

\paragraph{Short-context retrieval.}
REIGN is designed for long-document retrieval and underperforms
when inputs are shorter than the chunk size. On the two MTEB
short-text benchmarks we evaluated (\textsc{ArguAna},
\textsc{FiQA-2018}; \Cref{apx:mteb}), REIGN trails the truncated
GTE-small baseline by 5--7~nDCG@10 points. With a single chunk
embedding there is nothing to aggregate across, so the
cross-chunk signal REIGN is built to exploit disappears.

\paragraph{Lexically driven long-form text.}
A concrete in-distribution case appears on the legal subtasks of
LoCo (\Cref{ssec:loco}), where \texttt{courtlistener\_HTML} (26.15
nDCG@10), \texttt{courtlistener\_Plain} (24.89), and
\texttt{legal\_case\_reports} (32.69) trail REIGN's QA and
summarisation subtasks by as much as 74 points. These corpora are
dominated by highly
repetitive legal phrasing, so the retrieval signal is concentrated
in term-overlap rather than document-level semantics; a cross-chunk
encoder over pooled GN embeddings discards exactly the
surface-level lexical contrast that such corpora rely on. We view
this as a deliberate trade-off of the embedding-level paradigm;
hybrid configurations composing REIGN with a sparse lexical channel
are a natural direction for future work.

\paragraph{Coupling to the Guidance Network.}
REIGN's representation quality is upper-bounded by its frozen GN, so
any biases, semantic artefacts, or compression errors in the GN
embeddings are inherited unchanged. Switching to a new GN requires
full retraining of the REIGN encoder; the cached chunk-embedding
artefact is GN-specific. The encoder-capacity ablation
(\Cref{apx:ablation_indist}, \Cref{apx:ablation_ood_full}) further
shows that an in-distribution
``sweet-spot'' configuration can be misleading, because the
OOD-justified default differs from the in-distribution optimum.
Practitioners adapting REIGN to a new domain should plan an OOD
verification step.

\paragraph{Chunk-size $\times$ memory trade-off.}
Smaller chunks yield finer-grained signals but raise GPU memory
usage with chunk count. As a concrete reference, a
\textsc{REIGN-small-l2} model paired with GTE-large at effective
batch size 312 consumes approximately 1.8, 3.6, 4.5, and 8.4~GB
for chunk sizes 512, 256, 128, and 64 respectively.

\paragraph{Token-level fidelity.}
Because REIGN consumes abstracted chunk embeddings, it is not
suited to tasks that demand token-level fidelity (syntactic
parsing, numerical reasoning, fine-grained entity matching),
where the lossy compression at the GN boundary discards information
the downstream task cannot reconstruct.

\section*{Ethics Statement}

\textsc{GoodWiki-Long-Synthetic} derives from GoodWiki \citep{GoodWiki},
a cleaned release of English Wikipedia distributed under CC BY-SA; our
dataset release preserves that licensing and attribution. The rephrased
corpus documents are machine-generated with GPT-4o-mini
\citep{openai2024gpt4ocard} and are marked as synthetic in the dataset
card. \textsc{DAPFAM} \citep{ayaou2025dapfam} is built from public USPTO
patent records. No human subjects were involved, and neither corpus
contains private or personally identifying information beyond what is
already public in Wikipedia articles and patent filings.

\section*{Acknowledgements}

This work was supported by OBSS under project code ARGEM-024.


\bibliography{custom}

\begin{thebibliography}{43}
\providecommand{\natexlab}[1]{#1}

\bibitem[{Ayaou et~al.(2025)Ayaou, Cavallucci, and Chibane}]{ayaou2025dapfam}
Iliass Ayaou, Denis Cavallucci, and Hicham Chibane. 2025.
\newblock \href {https://arxiv.org/abs/2506.22141} {{DAPFAM}: A domain-aware
  family-level dataset to benchmark cross-domain patent retrieval}.
\newblock \emph{Preprint}, arXiv:2506.22141.

\bibitem[{Beltagy et~al.(2020)Beltagy, Peters, and
  Cohan}]{beltagy2020longformer}
Iz~Beltagy, Matthew~E Peters, and Arman Cohan. 2020.
\newblock \href {https://arxiv.org/abs/2004.05150} {Longformer: The
  long-document transformer}.
\newblock \emph{Preprint}, arXiv:2004.05150.

\bibitem[{Chen et~al.(2024)Chen, Xiao, Zhang, Luo, Lian, and
  Liu}]{chen-etal-2024-m3}
Jianlyu Chen, Shitao Xiao, Peitian Zhang, Kun Luo, Defu Lian, and Zheng Liu.
  2024.
\newblock \href {https://doi.org/10.18653/v1/2024.findings-acl.137}
  {{M}3-embedding: Multi-linguality, multi-functionality, multi-granularity
  text embeddings through self-knowledge distillation}.
\newblock In \emph{Findings of the Association for Computational Linguistics:
  ACL 2024}, pages 2318--2335, Bangkok, Thailand. Association for Computational
  Linguistics.

\bibitem[{Chen et~al.(2020)Chen, Kornblith, Norouzi, and
  Hinton}]{chen2020simple}
Ting Chen, Simon Kornblith, Mohammad Norouzi, and Geoffrey Hinton. 2020.
\newblock A simple framework for contrastive learning of visual
  representations.
\newblock In \emph{Proceedings of the 37th International Conference on Machine
  Learning}, ICML'20, pages 1597--1607. JMLR.org.

\bibitem[{Choi et~al.(2024)Choi, Kim, Lee, Kwon, Gu, Kim, Cho, and yong
  Sohn}]{choi2024linqembedmistraltechnicalreport}
Chanyeol Choi, Junseong Kim, Seolhwa Lee, Jihoon Kwon, Sangmo Gu, Yejin Kim,
  Minkyung Cho, and Jy~yong Sohn. 2024.
\newblock \href {https://arxiv.org/abs/2412.03223} {Linq-embed-mistral
  technical report}.
\newblock \emph{Preprint}, arXiv:2412.03223.

\bibitem[{Choi(2023)}]{GoodWiki}
Euirim Choi. 2023.
\newblock Goodwiki dataset.
\newblock \url{https://www.github.com/euirim/goodwiki}.

\bibitem[{Cieri et~al.(2004)Cieri, Miller, and Walker}]{cieri2004fisher}
Christopher Cieri, David Miller, and Kevin Walker. 2004.
\newblock \href {https://aclanthology.org/L04-1500/} {The {F}isher corpus: A
  resource for the next generations of speech-to-text}.
\newblock In \emph{Proceedings of the Fourth International Conference on
  Language Resources and Evaluation ({LREC}'04)}, Lisbon, Portugal. European
  Language Resources Association (ELRA).

\bibitem[{{Common Crawl Foundation}(2024)}]{commoncrawl_cc_main_2024_10}
{Common Crawl Foundation}. 2024.
\newblock Common crawl corpus: Cc-main-2024-10.
\newblock
  \url{https://data.commoncrawl.org/crawl-data/CC-MAIN-2024-10/index.html}.
\newblock Accessed: 2025-05-02.

\bibitem[{Dettmers et~al.(2023)Dettmers, Pagnoni, Holtzman, and
  Zettlemoyer}]{dettmers2023qlora}
Tim Dettmers, Artidoro Pagnoni, Ari Holtzman, and Luke Zettlemoyer. 2023.
\newblock {QL}o{RA}: Efficient finetuning of quantized {LLM}s.
\newblock In \emph{Proceedings of the 37th International Conference on Neural
  Information Processing Systems}, NIPS '23, pages 10088--10115, Red Hook, NY,
  USA. Curran Associates Inc.

\bibitem[{Ding et~al.(2020)Ding, Zhou, Yang, and Tang}]{ding2020cogltx}
Ming Ding, Chang Zhou, Hongxia Yang, and Jie Tang. 2020.
\newblock \href
  {https://proceedings.neurips.cc/paper_files/paper/2020/file/96671501524948bc3937b4b30d0e57b9-Paper.pdf}
  {{C}og{LTX}: Applying {BERT} to long texts}.
\newblock In \emph{Advances in Neural Information Processing Systems},
  volume~33, pages 12792--12804. Curran Associates, Inc.

\bibitem[{Faysse(2023)}]{faysse_project_gutenberg_2023}
Manuel Faysse. 2023.
\newblock Project gutenberg dataset.
\newblock \url{https://huggingface.co/datasets/manu/project_gutenberg}.
\newblock Accessed: 2025-05-02.

\bibitem[{Foundation(2025)}]{wikidump}
Wikimedia Foundation. 2025.
\newblock Wikimedia downloads.
\newblock \url{https://dumps.wikimedia.org}.

\bibitem[{Fu et~al.(2023)Fu, Arora, Grogan, Johnson, Eyuboglu, Thomas, Spector,
  Poli, Rudra, and R{\'e}}]{fu2023monarch}
Daniel~Y. Fu, Simran Arora, Jessica Grogan, Isys Johnson, Sabri Eyuboglu,
  Armin~W. Thomas, Benjamin Spector, Michael Poli, Atri Rudra, and Christopher
  R{\'e}. 2023.
\newblock {M}onarch {M}ixer: A simple sub-quadratic {GEMM}-based architecture.
\newblock In \emph{Proceedings of the 37th International Conference on Neural
  Information Processing Systems}, NIPS '23, pages 77546--77603, Red Hook, NY,
  USA. Curran Associates Inc.

\bibitem[{Houlsby et~al.(2019)Houlsby, Giurgiu, Jastrzebski, Morrone,
  De~Laroussilhe, Gesmundo, Attariyan, and Gelly}]{houlsby2019parameter}
Neil Houlsby, Andrei Giurgiu, Stanislaw Jastrzebski, Bruna Morrone, Quentin
  De~Laroussilhe, Andrea Gesmundo, Mona Attariyan, and Sylvain Gelly. 2019.
\newblock \href {https://proceedings.mlr.press/v97/houlsby19a.html}
  {Parameter-efficient transfer learning for {NLP}}.
\newblock In \emph{Proceedings of the 36th International Conference on Machine
  Learning}, volume~97 of \emph{Proceedings of Machine Learning Research},
  pages 2790--2799. PMLR.

\bibitem[{Hu et~al.(2022)Hu, Shen, Wallis, Allen-Zhu, Li, Wang, Wang, and
  Chen}]{hu2022lora}
Edward~J Hu, Yelong Shen, Phillip Wallis, Zeyuan Allen-Zhu, Yuanzhi Li, Shean
  Wang, Lu~Wang, and Weizhu Chen. 2022.
\newblock \href {https://openreview.net/forum?id=nZeVKeeFYf9} {Lo{RA}: Low-rank
  adaptation of large language models}.
\newblock In \emph{International Conference on Learning Representations}.

\bibitem[{Kaplan et~al.(2020)Kaplan, McCandlish, Henighan, Brown, Chess, Child,
  Gray, Radford, Wu, and Amodei}]{kaplan2020scaling}
Jared Kaplan, Sam McCandlish, Tom Henighan, Tom~B. Brown, Benjamin Chess, Rewon
  Child, Scott Gray, Alec Radford, Jeffrey Wu, and Dario Amodei. 2020.
\newblock \href {https://arxiv.org/abs/2001.08361} {Scaling laws for neural
  language models}.
\newblock \emph{Preprint}, arXiv:2001.08361.

\bibitem[{Lee et~al.(2025)Lee, Chen, Dua, Cer, Shanbhogue, Naim, {\'A}brego,
  Li, Chen, Vera, Ren, Zhang, Salz, Boratko, Han, Chen, Huang, Rao, Suganthan,
  Han, Doumanoglou, Gupta, Moiseev, Yip, Jain, Baumgartner, Shahi, {Palma
  Gomez}, Mariserla, Choi, Shah, Goenka, Chen, Xia, Chen, Duddu, Chen, Walker,
  Zhou, Ghiya, Gleicher, Gill, Dong, Seyedhosseini, Sung, Hoffmann, and
  Duerig}]{lee2025gemini}
Jinhyuk Lee, Feiyang Chen, Sahil Dua, Daniel Cer, Madhuri Shanbhogue, Iftekhar
  Naim, Gustavo~Hern{\'a}ndez {\'A}brego, Zhe Li, Kaifeng Chen,
  Henrique~Schechter Vera, Xiaoqi Ren, Shanfeng Zhang, Daniel Salz, Michael
  Boratko, Jay Han, Blair Chen, Shuo Huang, Vikram Rao, Paul Suganthan, and 28
  others. 2025.
\newblock \href {https://arxiv.org/abs/2503.07891} {Gemini embedding:
  Generalizable embeddings from {G}emini}.
\newblock \emph{Preprint}, arXiv:2503.07891.

\bibitem[{Li et~al.(2023{\natexlab{a}})Li, Yates, MacAvaney, He, and
  Sun}]{parade2023li}
Canjia Li, Andrew Yates, Sean MacAvaney, Ben He, and Yingfei Sun.
  2023{\natexlab{a}}.
\newblock \href {https://doi.org/10.1145/3600088} {{PARADE}: Passage
  representation aggregation for document reranking}.
\newblock \emph{ACM Trans. Inf. Syst.}, 42(2).

\bibitem[{Li et~al.(2023{\natexlab{b}})Li, Zhang, Zhang, Long, Xie, and
  Zhang}]{li2023towards}
Zehan Li, Xin Zhang, Yanzhao Zhang, Dingkun Long, Pengjun Xie, and Meishan
  Zhang. 2023{\natexlab{b}}.
\newblock \href {https://arxiv.org/abs/2308.03281} {Towards general text
  embeddings with multi-stage contrastive learning}.
\newblock \emph{Preprint}, arXiv:2308.03281.

\bibitem[{Liu et~al.(2024)Liu, Wang, Yin, Molchanov, Wang, Cheng, and
  Chen}]{liu2024dora}
Shih-Yang Liu, Chien-Yi Wang, Hongxu Yin, Pavlo Molchanov, Yu-Chiang~Frank
  Wang, Kwang-Ting Cheng, and Min-Hung Chen. 2024.
\newblock {DoRA}: Weight-decomposed low-rank adaptation.
\newblock In \emph{Proceedings of the 41st International Conference on Machine
  Learning}, ICML'24, pages 32100--32121. JMLR.org.

\bibitem[{Meng et~al.(2024)Meng, Liu, Joty, Xiong, Zhou, and
  Yavuz}]{SFRAIResearch2024}
Rui Meng, Ye~Liu, Shafiq~Rayhan Joty, Caiming Xiong, Yingbo Zhou, and Semih
  Yavuz. 2024.
\newblock \href {https://www.salesforce.com/blog/sfr-embedding/}
  {{SFR}-{E}mbedding-{M}istral: Enhance text retrieval with transfer learning}.
\newblock Salesforce AI Research Blog.

\bibitem[{Muennighoff et~al.(2023)Muennighoff, Tazi, Magne, and
  Reimers}]{muennighoff2022mteb}
Niklas Muennighoff, Nouamane Tazi, Lo{\"\i}c Magne, and Nils Reimers. 2023.
\newblock \href {https://doi.org/10.18653/v1/2023.eacl-main.148} {{MTEB}:
  Massive text embedding benchmark}.
\newblock In \emph{Proceedings of the 17th Conference of the European Chapter
  of the Association for Computational Linguistics}, pages 2014--2037,
  Dubrovnik, Croatia. Association for Computational Linguistics.

\bibitem[{Nussbaum et~al.(2025)Nussbaum, Morris, Duderstadt, and
  Mulyar}]{nussbaum2024nomic}
Zach Nussbaum, John~X. Morris, Brandon Duderstadt, and Andriy Mulyar. 2025.
\newblock \href {https://openreview.net/forum?id=IPmzyQSiQE} {{N}omic embed:
  Training a reproducible long context text embedder}.
\newblock \emph{Transactions on Machine Learning Research}.

\bibitem[{Oord et~al.(2018)Oord, Li, and Vinyals}]{oord2018representation}
Aaron van~den Oord, Yazhe Li, and Oriol Vinyals. 2018.
\newblock \href {https://arxiv.org/abs/1807.03748} {Representation learning
  with contrastive predictive coding}.
\newblock \emph{Preprint}, arXiv:1807.03748.

\bibitem[{OpenAI(2023)}]{tiktoken2023}
OpenAI. 2023.
\newblock tiktoken: A fast {BPE} tokenizer for use with {O}pen{AI}'s models.
\newblock \url{https://github.com/openai/tiktoken}.
\newblock Version 0.9.0.

\bibitem[{{OpenAI}(2024)}]{openai2024gpt4ocard}
{OpenAI}. 2024.
\newblock \href {https://arxiv.org/abs/2410.21276} {{GPT}-4o system card}.
\newblock \emph{Preprint}, arXiv:2410.21276.

\bibitem[{OpenAI()}]{tokencount2024openai}
OpenAI. 2024.
\newblock \href
  {https://help.openai.com/en/articles/4936856-what-are-tokens-and-how-to-count-them}
  {Understanding and counting tokens}.
\newblock OpenAI Help Center; accessed 2026-08-28.

\bibitem[{Pappagari et~al.(2019)Pappagari, Zelasko, Villalba, Carmiel, and
  Dehak}]{pappagari2019hierarchical}
Raghavendra Pappagari, Piotr Zelasko, Jes{\'u}s Villalba, Yishay Carmiel, and
  Najim Dehak. 2019.
\newblock Hierarchical transformers for long document classification.
\newblock In \emph{2019 IEEE Automatic Speech Recognition and Understanding
  Workshop (ASRU)}, pages 838--844. IEEE.

\bibitem[{Saad-Falcon et~al.(2024)Saad-Falcon, Fu, Arora, Guha, and
  R{\'e}}]{saad2024benchmarking}
Jon Saad-Falcon, Daniel~Y Fu, Simran Arora, Neel Guha, and Christopher R{\'e}.
  2024.
\newblock \href {https://proceedings.mlr.press/v235/saad-falcon24a.html}
  {Benchmarking and building long-context retrieval models with {L}o{C}o and
  {M2-BERT}}.
\newblock In \emph{Proceedings of the 41st International Conference on Machine
  Learning}, volume 235 of \emph{Proceedings of Machine Learning Research},
  pages 42918--42946. PMLR.

\bibitem[{Sennrich et~al.(2016)Sennrich, Haddow, and
  Birch}]{sennrich2015neural}
Rico Sennrich, Barry Haddow, and Alexandra Birch. 2016.
\newblock \href {https://doi.org/10.18653/v1/P16-1162} {Neural machine
  translation of rare words with subword units}.
\newblock In \emph{Proceedings of the 54th Annual Meeting of the Association
  for Computational Linguistics (Volume 1: Long Papers)}, pages 1715--1725,
  Berlin, Germany. Association for Computational Linguistics.

\bibitem[{Sturua et~al.(2025)Sturua, Mohr, Kalim~Akram, G{\"u}nther, Wang,
  Krimmel, Wang, Mastrapas, Koukounas, Wang, and Xiao}]{sturua2024jina}
Saba Sturua, Isabelle Mohr, Mohammad Kalim~Akram, Michael G{\"u}nther, Bo~Wang,
  Markus Krimmel, Feng Wang, Georgios Mastrapas, Andreas Koukounas, Nan Wang,
  and Han Xiao. 2025.
\newblock \href {https://doi.org/10.1007/978-3-031-88720-8_21} {{J}ina
  {E}mbeddings {V3}: Multilingual text encoder with low-rank adaptations}.
\newblock In \emph{Advances in Information Retrieval: 47th European Conference
  on Information Retrieval, ECIR 2025, Lucca, Italy, April 6--10, 2025,
  Proceedings, Part V}, pages 123--129, Berlin, Heidelberg. Springer-Verlag.

\bibitem[{Thakur et~al.(2021)Thakur, Reimers, R{\"u}ckl{\'e}, Srivastava, and
  Gurevych}]{thakur2021beir}
Nandan Thakur, Nils Reimers, Andreas R{\"u}ckl{\'e}, Abhishek Srivastava, and
  Iryna Gurevych. 2021.
\newblock \href {https://openreview.net/forum?id=wCu6T5xFjeJ} {{BEIR}: A
  heterogeneous benchmark for zero-shot evaluation of information retrieval
  models}.
\newblock In \emph{Thirty-fifth Conference on Neural Information Processing
  Systems Datasets and Benchmarks Track (Round 2)}.

\bibitem[{Vaswani et~al.(2017)Vaswani, Shazeer, Parmar, Uszkoreit, Jones,
  Gomez, Kaiser, and Polosukhin}]{vaswani2017attention}
Ashish Vaswani, Noam Shazeer, Niki Parmar, Jakob Uszkoreit, Llion Jones,
  Aidan~N Gomez, {\L}ukasz Kaiser, and Illia Polosukhin. 2017.
\newblock Attention is all you need.
\newblock \emph{Advances in Neural Information Processing Systems}, 30.

\bibitem[{Wang et~al.(2022)Wang, Yang, Huang, Jiao, Yang, Jiang, Majumder, and
  Wei}]{wang2022text}
Liang Wang, Nan Yang, Xiaolong Huang, Binxing Jiao, Linjun Yang, Daxin Jiang,
  Rangan Majumder, and Furu Wei. 2022.
\newblock \href {https://arxiv.org/abs/2212.03533} {Text embeddings by
  weakly-supervised contrastive pre-training}.
\newblock \emph{Preprint}, arXiv:2212.03533.

\bibitem[{Wang et~al.(2024)Wang, Yang, Huang, Yang, Majumder, and
  Wei}]{wang-etal-2024-improving-text}
Liang Wang, Nan Yang, Xiaolong Huang, Linjun Yang, Rangan Majumder, and Furu
  Wei. 2024.
\newblock \href {https://doi.org/10.18653/v1/2024.acl-long.642} {Improving text
  embeddings with large language models}.
\newblock In \emph{Proceedings of the 62nd Annual Meeting of the Association
  for Computational Linguistics (Volume 1: Long Papers)}, pages 11897--11916,
  Bangkok, Thailand. Association for Computational Linguistics.

\bibitem[{Wu et~al.(2016)Wu, Schuster, Chen, Le, Norouzi, Macherey, Krikun,
  Cao, Gao, Macherey, Klingner, Shah, Johnson, Liu, Kaiser, Gouws, Kato, Kudo,
  Kazawa, Stevens, Kurian, Patil, Wang, Young, Smith, Riesa, Rudnick, Vinyals,
  Corrado, Hughes, and Dean}]{wu2016google}
Yonghui Wu, Mike Schuster, Zhifeng Chen, Quoc~V. Le, Mohammad Norouzi, Wolfgang
  Macherey, Maxim Krikun, Yuan Cao, Qin Gao, Klaus Macherey, Jeff Klingner,
  Apurva Shah, Melvin Johnson, Xiaobing Liu, {\L}ukasz Kaiser, Stephan Gouws,
  Yoshikiyo Kato, Taku Kudo, Hideto Kazawa, and 12 others. 2016.
\newblock \href {https://arxiv.org/abs/1609.08144} {Google's neural machine
  translation system: Bridging the gap between human and machine translation}.
\newblock \emph{Preprint}, arXiv:1609.08144.

\bibitem[{Xiao et~al.(2024)Xiao, Liu, Zhang, Muennighoff, Lian, and
  Nie}]{xiao2024cpack}
Shitao Xiao, Zheng Liu, Peitian Zhang, Niklas Muennighoff, Defu Lian, and
  Jian-Yun Nie. 2024.
\newblock \href {https://doi.org/10.1145/3626772.3657878} {{C}-{P}ack: Packed
  resources for general {C}hinese embeddings}.
\newblock In \emph{Proceedings of the 47th International ACM SIGIR Conference
  on Research and Development in Information Retrieval}, SIGIR '24, pages
  641--649, New York, NY, USA. Association for Computing Machinery.

\bibitem[{Yang et~al.(2020)Yang, Zhang, Li, Bendersky, and
  Najork}]{yang2020smith}
Liu Yang, Mingyang Zhang, Cheng Li, Michael Bendersky, and Marc Najork. 2020.
\newblock \href {https://doi.org/10.1145/3340531.3411908} {Beyond 512 tokens:
  {S}iamese multi-depth transformer-based hierarchical encoder for long-form
  document matching}.
\newblock In \emph{Proceedings of the 29th {ACM} International Conference on
  Information \& Knowledge Management}, CIKM '20, pages 1725--1734, New York,
  NY, USA. Association for Computing Machinery.

\bibitem[{Yang et~al.(2018)Yang, Qi, Zhang, Bengio, Cohen, Salakhutdinov, and
  Manning}]{yang2018hotpotqa}
Zhilin Yang, Peng Qi, Saizheng Zhang, Yoshua Bengio, William Cohen, Ruslan
  Salakhutdinov, and Christopher~D. Manning. 2018.
\newblock \href {https://doi.org/10.18653/v1/D18-1259} {{H}otpot{QA}: A dataset
  for diverse, explainable multi-hop question answering}.
\newblock In \emph{Proceedings of the 2018 Conference on Empirical Methods in
  Natural Language Processing}, pages 2369--2380, Brussels, Belgium.
  Association for Computational Linguistics.

\bibitem[{Zaheer et~al.(2020)Zaheer, Guruganesh, Dubey, Ainslie, Alberti,
  Ontanon, Pham, Ravula, Wang, Yang, and Ahmed}]{zaheer2020big}
Manzil Zaheer, Guru Guruganesh, Avinava Dubey, Joshua Ainslie, Chris Alberti,
  Santiago Ontanon, Philip Pham, Anirudh Ravula, Qifan Wang, Li~Yang, and Amr
  Ahmed. 2020.
\newblock Big {B}ird: Transformers for longer sequences.
\newblock In \emph{Proceedings of the 34th International Conference on Neural
  Information Processing Systems}, NIPS '20, pages 17283--17297, Red Hook, NY,
  USA. Curran Associates Inc.

\bibitem[{Zhang et~al.(2024)Zhang, Li, Zeng, and Wang}]{stella2024}
Dun Zhang, Jiacheng Li, Ziyang Zeng, and Fulong Wang. 2024.
\newblock \href {https://arxiv.org/abs/2412.19048} {Jasper and {S}tella:
  Distillation of {SOTA} embedding models}.
\newblock \emph{Preprint}, arXiv:2412.19048.

\bibitem[{Zhang et~al.(2025)Zhang, Li, Long, Zhang, Lin, Yang, Xie, Yang, Liu,
  Lin, Huang, and Zhou}]{zhang2025qwen3embeddingadvancingtext}
Yanzhao Zhang, Mingxin Li, Dingkun Long, Xin Zhang, Huan Lin, Baosong Yang,
  Pengjun Xie, An~Yang, Dayiheng Liu, Junyang Lin, Fei Huang, and Jingren Zhou.
  2025.
\newblock \href {https://arxiv.org/abs/2506.05176} {Qwen3 embedding: Advancing
  text embedding and reranking through foundation models}.
\newblock \emph{Preprint}, arXiv:2506.05176.

\bibitem[{Zhu et~al.(2015)Zhu, Kiros, Zemel, Salakhutdinov, Urtasun, Torralba,
  and Fidler}]{zhu2015aligning}
Yukun Zhu, Ryan Kiros, Rich Zemel, Ruslan Salakhutdinov, Raquel Urtasun,
  Antonio Torralba, and Sanja Fidler. 2015.
\newblock \href {https://doi.org/10.1109/ICCV.2015.11} {Aligning books and
  movies: Towards story-like visual explanations by watching movies and reading
  books}.
\newblock In \emph{Proceedings of the 2015 IEEE International Conference on
  Computer Vision (ICCV)}, ICCV '15, pages 19--27, USA. IEEE Computer Society.

\end{thebibliography}

\appendix
\crefalias{section}{appendix}

\section{MTEB Short-Context Generalisation}\label{apx:mteb}

\Cref{tab:mteb} reports retrieval performance on two MTEB
benchmarks (\textsc{ArguAna} and \textsc{FiQA-2018}) consisting
of short queries paired with comparatively short documents.
REIGN underperforms the truncated GTE-small baseline by
5.1~nDCG@10 on \textsc{ArguAna} and 6.8~nDCG@10 on
\textsc{FiQA-2018}. The drop is not a GN artefact, since the
truncated baseline sees the very same single chunk embedding. It
comes from the REIGN encoder, which is trained on multi-chunk
sequences: applied to a lone chunk embedding it has nothing to
aggregate, and the cross-chunk mechanism it relies on cannot
contribute. REIGN targets long-context retrieval, and short-passage
tasks lie outside the intended operating regime.

\begin{table*}[ht]
\centering
\caption{\textbf{Selected MTEB IR results.}
Baselines are evaluated with truncation (512 tokens).
REIGN uses the paper-default \textsc{base-l3} encoder (22M) with
\textsc{GTE-small} as the Guidance Network.
$\Delta$ is computed with respect to the truncated GTE-small
baseline. \textsc{ArguAna} provides only a test split.\\
\textsuperscript{\textdagger}Trained only on GoodWiki-Long; no
MTEB-specific fine-tuning.}
\label{tab:mteb}
\scriptsize
\begin{tabular}{lrrrrrr}
\toprule
\textbf{Model} & \textbf{MAP@1} & \textbf{MAP@10} & \textbf{R@10} &
\textbf{P@10} & \textbf{nDCG@10} & \textbf{$\Delta$ nDCG@10} \\
\midrule
\multicolumn{7}{l}{\textbf{ArguAna}}\\
GTE-small (baseline)
& \textbf{30.73} & \textbf{46.60} & \textbf{83.57} & \textbf{8.36} & \textbf{55.42} & -- \\
REIGN + GTE-small\textsuperscript{\textdagger}
& 26.53 & 41.52 & 78.66 & 7.87 & 50.32 & $-5.10$ \\[2pt]
\midrule
\multicolumn{7}{l}{\textbf{FiQA-2018}}\\
GTE-small (baseline)
& \textbf{19.34} & \textbf{31.72} & \textbf{45.63} & \textbf{11.16} & \textbf{39.31} & -- \\
REIGN + GTE-small\textsuperscript{\textdagger}
& 15.23 & 25.66 & 38.14 & 9.35 & 32.55 & $-6.76$ \\
\bottomrule
\end{tabular}
\end{table*}

\section{Encoder-Capacity Ablation: In-Distribution Sweep}\label{apx:ablation_indist}

\Cref{tab:ablation_indist} reports the full 4$\times$3 sweep of
REIGN encoder configurations against GTE Guidance Networks on the
\textsc{GoodWiki-Long} test set, evaluated at top-$k$~$=$~10. The
$\Delta$ column makes the inverted-U pattern immediately readable.
Under-capacity \textsc{tiny-l1} sits below the bare-GN baseline on all
three GNs; over-parameterised \textsc{large-l4} drops below baseline on
\textsc{GTE-base} and \textsc{GTE-large};
\textsc{small-l2} reaches the per-GN maximum on \textsc{GTE-small}
and \textsc{GTE-large}, with \textsc{base-l3} essentially tied on
\textsc{GTE-base} (visualised in \Cref{fig:ablation}, left).
Read in isolation, this would make \textsc{small-l2} the
paper default. The out-of-distribution head-to-head in
\Cref{apx:ablation_ood_full} reverses that conclusion, since
\textsc{base-l3} wins 11 of 12 cells (the 12th is a $-0.20$
near-tie), with average margins of $+0.71$~nDCG@100 on DAPFAM and
$+0.59$~nDCG@10 on LoCo. The in-distribution \textsc{small-l2}
optimum is therefore an overfitting artefact; \textsc{base-l3}
generalises better under cross-domain transfer, the regime
the paper targets, so an in-distribution capacity ablation
alone is insufficient. We adopt \textsc{base-l3} as the
paper-default REIGN encoder throughout the headline tables.

\begin{figure*}[t]
  \centering
  \includegraphics[width=\linewidth]{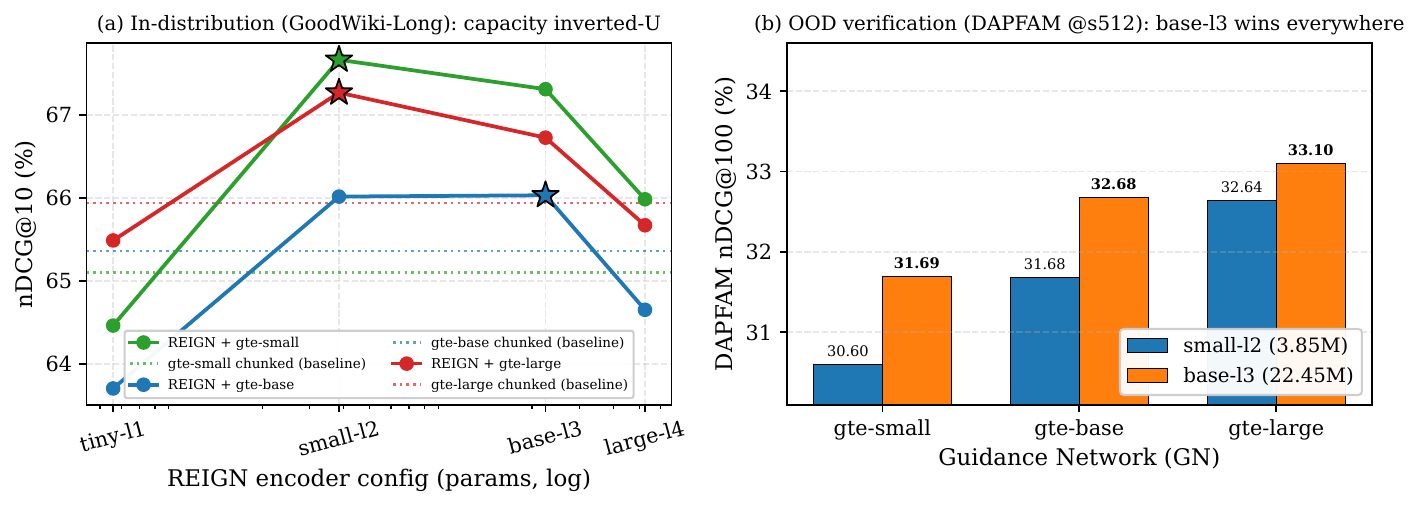}
  \caption{\textbf{Encoder-capacity sweet-spot reverses between
  in-distribution and OOD.} The left panel plots REIGN nDCG@10 on
  \textsc{GoodWiki-Long} vs.\ encoder size, per GN; stars mark per-GN
  maxima, dotted lines are bare-GN baselines. The right panel compares
  DAPFAM @s512 \textsc{small-l2} against \textsc{base-l3}, which wins
  at every GN scale. Full per-cell numbers in \Cref{apx:ablation_ood_full}.}
  \label{fig:ablation}
\end{figure*}

\begin{table*}[ht]
\centering
\caption{\textbf{REIGN config $\times$ GN ablation on \textsc{GoodWiki-Long} test (in-distribution).}
nDCG@10 ($\times$100); $\Delta$ measured against the per-GN bare
chunked-mean-pool baseline. Head dimension is held at $d/H{=}64$ across
configs.\\
$L$/$d$/$H$/FFN = the REIGN encoder's number of Transformer layers,
hidden width, attention heads, and feed-forward dimension.\\
\textbf{Bold} = best REIGN config per GN.}
\label{tab:ablation_indist}
\scriptsize
\setlength{\tabcolsep}{4pt}
\begin{tabular}{lrrrrrrrr}
\toprule
\textbf{REIGN config} & \textbf{Params} & \textbf{$L$} & \textbf{$d$} & \textbf{$H$} & \textbf{FFN} & \textbf{GTE-small ($\Delta$)} & \textbf{GTE-base ($\Delta$)} & \textbf{GTE-large ($\Delta$)} \\
\midrule
GN-only baseline             & --      & --  & --   & --  & --    & 65.10                    & 65.36                    & 65.94 \\
\midrule
\textsc{tiny-l1}             & 0.56M   & 1   & 192  & 3   & 768   & 64.46 ($-0.64$)          & 63.70 ($-1.66$)          & 65.49 ($-0.45$) \\
\textsc{small-l2}            & 3.85M   & 2   & 384  & 6   & 1536  & \textbf{67.67} ($+2.57$) & 66.02 ($+0.66$)          & \textbf{67.27} ($+1.33$) \\
\textsc{base-l3}             & 22.45M  & 3   & 768  & 12  & 3072  & 67.31 ($+2.21$)          & \textbf{66.03} ($+0.67$) & 66.73 ($+0.79$) \\
\textsc{large-l4}            & 52.49M  & 4   & 1024 & 16  & 4096  & 65.98 ($+0.88$)          & 64.65 ($-0.71$)          & 65.67 ($-0.27$) \\
\bottomrule
\end{tabular}
\end{table*}

\section{Encoder-Capacity Ablation: OOD Head-to-Head Detail}\label{apx:ablation_ood_full}

\Cref{tab:ablation_ood} reports the full per-cell out-of-distribution
head-to-head between \textsc{small-l2} and \textsc{base-l3} on DAPFAM
and LoCo at both strides; the bar-chart panel in
\Cref{fig:ablation} (\Cref{apx:ablation_indist}) visualises the
DAPFAM @s512 slice. \textsc{base-l3} wins 11 of 12 cells (the 12th is a $-0.20$
near-tie, LoCo at \textsc{GTE-large} stride 384); average
margins are $+0.71$~nDCG@100 on DAPFAM and $+0.59$~nDCG@10 on LoCo.

\begin{table}[ht]
\centering
\caption{\textbf{Out-of-distribution head-to-head: \textsc{small-l2} vs.\ \textsc{base-l3}.}
DAPFAM nDCG@100 (test split); LoCo macro-avg nDCG@10 (12 subtasks).
\textsc{base-l3} wins 11 of 12 cells. Deltas are computed on
unrounded scores.\\
\textbf{Bold} = winner per cell.}
\label{tab:ablation_ood}
\scriptsize
\setlength{\tabcolsep}{4pt}
\begin{tabular}{llrrr}
\toprule
\textbf{Benchmark} & \textbf{GN, stride} & \textbf{small-l2} & \textbf{base-l3} & \textbf{$\Delta$} \\
\midrule
DAPFAM & GTE-small,384  & 30.68          & \textbf{31.10}  & $+0.42$ \\
DAPFAM & GTE-small,512  & 30.60          & \textbf{31.69}  & $+1.10$ \\
DAPFAM & GTE-base,384  & 31.92          & \textbf{32.72}  & $+0.80$ \\
DAPFAM & GTE-base,512  & 31.68          & \textbf{32.68}  & $+1.00$ \\
DAPFAM & GTE-large,384  & 32.55          & \textbf{33.06}  & $+0.51$ \\
DAPFAM & GTE-large,512  & 32.64          & \textbf{33.10}  & $+0.45$ \\
\midrule
LoCo   & GTE-small,384  & 68.03          & \textbf{68.87}  & $+0.84$ \\
LoCo   & GTE-small,512  & 67.72          & \textbf{68.92}  & $+1.20$ \\
LoCo   & GTE-base,384  & 66.97          & \textbf{67.74}  & $+0.77$ \\
LoCo   & GTE-base,512  & 66.84          & \textbf{67.72}  & $+0.88$ \\
LoCo   & GTE-large,384  & \textbf{70.97} & 70.77           & $-0.20$ \\
LoCo   & GTE-large,512  & 70.42          & \textbf{70.49}  & $+0.07$ \\
\bottomrule
\end{tabular}
\end{table}

\section{Positional-Encoding Ablation}\label{apx:pe_ablation}

We ablate the chunk-position signal referenced in \Cref{ssec:framework},
comparing no positional encoding (our design) against learned absolute
chunk positions and against fixed sinusoidal encodings, both injected
after the GN projection. We train all three arms from scratch under one
protocol, using \textsc{base-l3} on \textsc{GTE-small},
chunk/stride 512, InfoNCE ($\tau{=}0.07$, partial weight 0.5), batch
48 with 47 in-batch negatives, 20 epochs, lr $10^{-5}$, seed 42, and
best-validation selection. The arms sit below the released operating point
(\Cref{tab:gw_lcm_pareto}) because the released checkpoints use the
cosine recipe of \Cref{ssec:objective}.

\begin{table}[ht]
\centering
\caption{\textbf{Chunk-position ablation.} nDCG@10 on
\textsc{GoodWiki-Long} test and LoCo (macro over 12 subtasks),
nDCG@100 on \textsc{DAPFAM} test; deltas vs.\ the no-PE arm, computed
on unrounded scores.}
\label{tab:pe_ablation}
\scriptsize
\setlength{\tabcolsep}{2.5pt}
\begin{tabular}{lrrr}
\toprule
\textbf{Chunk-position signal} & \textbf{GoodWiki} & \textbf{LoCo} & \textbf{DAPFAM} \\
\midrule
None (our design)       & \textbf{64.52} & \textbf{65.25} & \textbf{30.46} \\
Learned absolute        & 64.45 ($-0.08$) & 65.12 ($-0.13$) & 29.88 ($-0.58$) \\
Sinusoidal              & 62.36 ($-2.16$) & 59.48 ($-5.77$) & 21.40 ($-9.06$) \\
\bottomrule
\end{tabular}
\end{table}

\Cref{tab:pe_ablation} supports the no-PE design. Learned absolute
positions are near parity on all three benchmarks, and sinusoidal
encodings are worse everywhere; positional encoding adds parameters
without improving any benchmark.
Unit tests give an independent check. Without a positional signal the
pooled document representation is invariant to chunk order (maximum
deviation $3{\times}10^{-8}$), while the other two arms are
order-sensitive.

\section{Theoretical Analysis: Information Loss and Recovery}\label{apx:theory}

We formalise the information lost in the move from sub-word tokens to
GN embeddings, and separate that fixed loss from the trainable remainder.

\paragraph{Scope.} Theorem~\ref{thm:info_loss} assumes nothing about $\theta$
or about the objective that produced it, so it covers every REIGN
variant, the released checkpoints included, and Proposition~\ref{prop:sufficiency}
states when the trainable part of the gap closes. Only the recovery
bound of Proposition~\ref{prop:nce} depends on the objective. It covers the
InfoNCE-trained arms of \Cref{apx:pe_ablation,apx:loss_ablation},
whose in-batch negatives are draws from the positive marginal as the
proposition assumes; the \textsc{DAPFAM} fine-tuning runs
(\Cref{apx:dapfam}) also use InfoNCE, but their provided corpus
negatives and false-negative masking sit outside that sampling
assumption. The three-way
cosine objective of the released checkpoints (\Cref{ssec:objective})
admits no such bound; Remark~\ref{rem:cosine} states what it does give, with
the empirical case in \Cref{apx:loss_ablation}.

\begin{theorem}[Embedding Refurbishment]\label{thm:info_loss}
Let $X$ be a random variable over natural-language texts, $T=f(X)$ the
(deterministic\footnote{For a fixed tokeniser and vocabulary.}) sub-word
token sequence with $H(T)<\infty$ (which holds whenever document
length is bounded), $E=g(T)$ the frozen GN embedding sequence, and
$R_\theta=h_\theta(E)$ the fixed-size embedding produced by a REIGN encoder
$h_\theta$ with trainable parameters~$\theta$. Because $f$ and $g$ are
deterministic, $X\!\to\!T\!\to\!E\!\to\!R_\theta$ is a Markov chain and
$T$, $E$, $R_\theta$ take countably many values. For every $\theta$ and
every measurable $h_\theta$:

\begin{enumerate}[label=(\roman*)]
\item \emph{(GN coarsening.)} $I(X;E)\le I(X;T)$, and the gap is
      \[
      I(X;T)-I(X;E)\;=\;H(T\mid E),
      \]
      which vanishes iff $g$ is injective on the support of $T$.
\item \emph{(End-to-end loss decomposition.)}
      The total loss in mapping $X$ to $R_\theta$ decomposes as
      \begin{multline*}
        I(X;T)-I(X;R_\theta) \;=\; \\
        \underbrace{I(X;T)-I(X;E)}_{\text{GN coarsening (fixed)}} \\
        {}+\underbrace{I(X;E)-I(X;R_\theta)}_{\text{closed by training}},
      \end{multline*}
      with both terms non-negative. The first is an intrinsic property
      of the frozen GN, fixed at the value in~(i). The second equals
      $I(X;E\mid R_\theta)$ and is the only term $\theta$ moves.
\end{enumerate}
\end{theorem}

\begin{proof}
\leavevmode
\begin{enumerate}[label=(\roman*)]
\item $T$ and $E$ are deterministic functions of $X$, so $I(X;T)=H(T)$
      and $I(X;E)=H(E)$, both finite, and $E=g(T)$ gives $H(T,E)=H(T)$.
      Hence $I(X;T)-I(X;E)=H(T)-H(E)=H(T\mid E)\ge 0$. For countably
      valued $T$, $H(T\mid E)=0$ iff $T$ is almost surely a function of
      $E$, i.e.\ iff token sequences of positive probability receive
      distinct GN embeddings; collisions on the rest are harmless.
\item Both gaps are non-negative by the data-processing inequality along
      $X\!\to\!T\!\to\!E$ and $X\!\to\!E\!\to\!R_\theta$, and adding and
      subtracting $I(X;E)$ in $I(X;T)-I(X;R_\theta)$ gives the
      decomposition. Since $R_\theta$ is a function of $E$ we have
      $I(X;R_\theta\mid E)=0$, so expanding $I(X;E,R_\theta)$ by the
      chain rule in both orders leaves
      $I(X;E)-I(X;R_\theta)=I(X;E\mid R_\theta)$.
\end{enumerate}
\end{proof}

\begin{proposition}[Sufficiency closes the trainable gap]\label{prop:sufficiency}
Under the assumptions of Theorem~\ref{thm:info_loss}, the trainable gap
$I(X;E)-I(X;R_\theta)$ vanishes iff $X \perp\!\!\!\perp E \mid R_\theta$,
that is, iff $R_\theta$ is a sufficient statistic of $E$ for $X$, and in
particular whenever $h_\theta$ is injective on the support of $E$.
\end{proposition}

\begin{proof}
The gap equals $I(X;E\mid R_\theta)$ by Theorem~\ref{thm:info_loss}(ii), which
is zero exactly under that conditional independence. Injectivity on the
support of $E$ makes $E$ almost surely a function of $R_\theta$, so
$0\le I(X;E\mid R_\theta)\le H(E\mid R_\theta)=0$.
\end{proof}

\begin{proposition}[Contrastive recovery under InfoNCE]\label{prop:nce}
Let $(E,E^+)$ be a positive pair under a stochastic augmentation
$t\sim\mathcal{T}$ and let $\mathcal{L}_{\mathrm{NCE}}(\theta)$ be the
InfoNCE objective over batches of $B>1$ pairs drawn i.i.d., each anchor
scored against its own positive and $B-1$ independent draws from the
marginal of $E^+$, with critic
$(e^+,r)\mapsto\operatorname{sim}(h_\theta(e^+),r)/\tau$. Then
\[
    I(E^+;R_\theta)\;\ge\;\log B\;-\;\mathcal{L}_{\mathrm{NCE}}(\theta)
\]
for every $\theta$ \citep[Appendix~A.1]{oord2018representation}.
\end{proposition}

The bound applies at every $\theta$, not only at the optimum, and
tightens as the loss falls, but it certifies at most $\log B$ of
mutual information.

\begin{remark}[What the cosine objective gives]\label{rem:cosine}
The released recipe admits no analogue of Proposition~\ref{prop:nce}, since its
per-pair terms are not normalised over a candidate set and so support
no likelihood-ratio bound on mutual information. It does bear on the
condition of Proposition~\ref{prop:sufficiency}. Consider two documents with
nonzero embeddings $e\neq e'$ that appear together as an
$s{=}{-}1$ pair. Their cost
$c_+=\max\bigl(0,\cos(h_\theta(e),h_\theta(e'))\bigr)$ is at most $1$,
with equality iff the two embeddings point the same way, the collapse
$h_\theta(e)=h_\theta(e')$ included, and strictly below $1$ otherwise.
The negative branch thus penalises, on the negative pairs seen in
training, the collapses that violate the injectivity condition of
Proposition~\ref{prop:sufficiency}. We claim only that
direction of pressure, with no bound on $I(X;E\mid R_\theta)$ and no
property of the minimisers.
\end{remark}

\paragraph{Interpretation.}
Statement~(i) prices the swap of token space for GN embeddings at
$H(T\mid E)$, an \emph{irreversible} loss that no choice of $\theta$
recovers once the GN is fixed. Statement~(ii) leaves the remainder to
training, and Proposition~\ref{prop:sufficiency} names its target. REIGN cannot
recover what the GN discarded, but \emph{can} in principle preserve
everything $E$ carries about $X$, the formal counterpart of the
empirical claim that REIGN's quality is upper-bounded by its GN
(Limitations, ``Coupling to the Guidance Network''). Our two objectives differ in what supports that
``in principle''. InfoNCE brings the quantitative handle of
Proposition~\ref{prop:nce}; the released cosine recipe brings the qualitative
pressure of Remark~\ref{rem:cosine} and the benchmark comparison of
\Cref{apx:loss_ablation}, where it beats every InfoNCE arm we ran.

\section{Efficiency Analysis}\label{apx:efficiency}

\Cref{tab:flops_appendix_8k} reports concrete training and inference FLOPs for representative long-document embedding models under a fixed 8K-token setting with 512-token chunks ($k{=}16$).
For token-level Transformer models, we approximate the cost of a standard encoder layer as
$\mathcal{O}(16nd^2 + 2n^2d)$, where $n$ is the token sequence length, accounting for linear projections in the attention and feed-forward blocks as well as the quadratic self-attention term, following established practices in prior efficiency analyses~\citep{vaswani2017attention,kaplan2020scaling}.
For Longformer-style sparse attention, the quadratic term is reduced to $\mathcal{O}(2nwd)$, where $w$ denotes the local attention window size~\citep{beltagy2020longformer}.

\begin{table*}[ht]
\centering
\caption{\textbf{Concrete FLOPs comparison for long-document embedding methods (per document)} under an 8K-token setting with 512-token chunks ($k{=}16$).
Token-level baselines run self-attention over subword tokens; REIGN
also runs full self-attention but over chunk-level embeddings
(sequence length $k{=}16$ instead of $n{=}512$ or $8K$), and caches
the GN forward pass.
Multiply–add operations count as 2 FLOPs; training FLOPs assume $3{\times}$ forward passes (forward + backward) plus the contrastive-objective overhead at batch size $B{=}48$.}
\label{tab:flops_appendix_8k}
\scriptsize
\begin{tabular}{llcccc}
\toprule
\textbf{Model} & \textbf{Attention Type} & \textbf{Representation} & \textbf{Training TFLOPs} & \textbf{Inference TFLOPs} & \textbf{Scaling} \\
\midrule
GTE (truncated, 512)
& Full self-attn
& Token-level
& 0.657
& 0.219
& $\mathcal{O}(n^2)$ \\

GTE (chunked, $16{\times}512$)
& Full self-attn (repeated)
& Token-level
& 10.514
& 3.505
& $\mathcal{O}(kn^2)$ \\

Longformer (full 8K)
& Sparse windowed ($w{=}512$)
& Token-level
& 10.509
& 3.503
& $\mathcal{O}(nw)$ \\

Full-attention encoder (8K)
& Full self-attn
& Token-level
& 19.785
& 6.595
& $\mathcal{O}(n^2)$ \\

\midrule
REIGN (small-l2, \textit{no cache})
& Full self-attn
& Embedding-level
& 10.514
& 3.505
& $\mathcal{O}(kn^2)$ \\

\textbf{REIGN (small-l2, cached GN)}
& Full self-attn
& Embedding-level
& \textbf{0.00027}
& \textbf{0.000088}
& $\boldsymbol{\mathcal{O}(k^2)}$ \\

\bottomrule
\end{tabular}
\end{table*}

All FLOPs counts treat a multiply–add as two floating-point operations.
Training FLOPs are computed as three times the forward-pass cost to account for forward and backward propagation, with additional overhead from the contrastive objective at batch size $B{=}48$; this $B$ is a fixed analytic setting applied to all systems, not the training batch size of \Cref{apx:goodwiki_train}.
Unless otherwise stated, GTE-large, Longformer-large, and the full-attention 8K encoder assume hidden dimension $d{=}1024$ and $L{=}24$ encoder layers.
REIGN pairs a frozen GTE Guidance Network with a $1024{\rightarrow}384$ projection and a lightweight two-layer (\textsc{small-l2}) encoder over $k{=}16$ chunk-level embeddings.
When Guidance Network embeddings are cached, REIGN avoids repeated token-level computation entirely, resulting in several orders-of-magnitude reductions in both training and inference FLOPs relative to token-level long-context baselines.

\paragraph{Measured cost.}
\Cref{tab:measured_efficiency} complements the analytic model with
measured end-to-end latency and peak memory on a single RTX~4090
(24\,GB). We time 500 corpus documents and 100 queries from the
\textsc{GoodWiki-Long} test split at batch 8, with one warm-up plus
three timed repeats on an otherwise-idle GPU. Cached REIGN serves GN chunk
embeddings from disk and answers queries in 0.40--0.52\,ms
($49$--$229\times$ faster than with the GN run per query, i.e.\ the
uncached rows); the one-time cache build costs
8.9\,ms (GTE-small), 18.0\,ms (GTE-base), and 53.3\,ms (GTE-large)
per document. Uncached REIGN performs identical GN work to the
chunked baseline (verified at 167 chunks in both paths on the sampled
documents), so we report parity; the sub-1.0 ratios reflect baseline
batching, not less work. Two caveats remain. The cached peak-memory
figures still include the GN weights resident on the device, which a
corpus-only indexing deployment would not need, and Jina-v3 could not
run at batch 8 on this card (measured at batch 4, peaking at
18.9\,GB).

\begin{table}[t]
\centering
\caption{\textbf{Measured inference cost} on \textsc{GoodWiki-Long}
(500 documents, 100 queries), single RTX~4090 (24\,GB),
otherwise-idle GPU. Latency is end-to-end per query (encode +
retrieve).\\
\emph{Uncached} runs the Guidance Network per query; \emph{cached}
serves its chunk embeddings from disk.\\
Jina-v3 measured at batch 4 (batch 8 exceeds
24\,GB); all other dense rows at batch 8.}
\label{tab:measured_efficiency}
\scriptsize
\setlength{\tabcolsep}{4pt}
\begin{tabular}{lrrr}
\toprule
\textbf{System} & \textbf{Index (s)} & \textbf{ms/query} & \textbf{Peak GPU (GB)} \\
\midrule
\multicolumn{4}{l}{\emph{Sparse lexical}}\\
BM25 & 0.2 & 6.5 & -- \\
TF-IDF & 0.1 & 1.1 & -- \\
\multicolumn{4}{l}{\emph{Native long-context dense (truncate 8K)}}\\
BGE-M3 & 37.6 & 430.3 & 9.80 \\
Jina-Embeddings-v3 & 26.0 & 336.7 & 18.90 \\
Stella-en-1.5B-v5 & 30.6 & 318.9 & 14.34 \\
Nomic-Embed-v1.5 & 13.4 & 160.5 & 4.81 \\
\multicolumn{4}{l}{\emph{Bare GN (chunked mean-pool, 512)}}\\
GTE-small (chunked) & 4.1 & 23.0 & 0.23 \\
GTE-base (chunked) & 7.7 & 42.7 & 0.61 \\
GTE-large (chunked) & 21.6 & 121.4 & 1.56 \\
BGE-base (chunked) & 7.7 & 42.9 & 0.61 \\
BGE-large (chunked) & 21.6 & 121.4 & 1.56 \\
\multicolumn{4}{l}{\emph{REIGN, uncached GN (cold cache)}}\\
REIGN + GTE-small & 3.3 & 19.8 & 0.35 \\
REIGN + GTE-base & 6.6 & 39.7 & 0.76 \\
REIGN + GTE-large & 19.9 & 118.2 & 1.73 \\
\multicolumn{4}{l}{\emph{REIGN, cached GN embeddings (warm)}}\\
REIGN + GTE-small & 0.2 & 0.40 & 0.24 \\
REIGN + GTE-base & 0.2 & 0.47 & 0.55 \\
REIGN + GTE-large & 0.2 & 0.52 & 1.45 \\
\bottomrule
\end{tabular}
\end{table}

\section{GoodWiki-Long Training Protocol}\label{apx:goodwiki_train}

\paragraph{Optimisation.} REIGN is trained on \textsc{GoodWiki-Long}
with the three-way cosine embedding objective of \Cref{ssec:objective}
(partial weight $\lambda{=}0.5$). Optimiser AdamW with
$\text{lr}{=}10^{-5}$ and weight decay $10^{-4}$; cosine-annealing
schedule. Batch size is 18. Each anchor is paired with its rephrased
positive (target 1), its two topical distractors as partials
(target 0), and 17 in-batch negatives formed by shifting the positives
(target $-1$), giving 360 pairs per optimisation step. Training runs
50 epochs with validation every 4 epochs; the released checkpoints are
the best-validation snapshots (nDCG@10 on the \texttt{val} split). GN
chunk embeddings are precomputed and cached. 16-bit mixed precision,
seed 42.

\paragraph{Encoder configuration.} The paper-default
\textsc{base-l3} REIGN encoder is a 3-layer Transformer over chunk-level
embeddings (22.45M parameters; see \Cref{tab:ablation_indist} for the
full configuration sweep). An internal linear projection adapts the
GN's chunk-embedding dimensionality to the REIGN encoder hidden width
when the two differ (e.g.\ $1024{\rightarrow}768$ for
\textsc{GTE-large}). Chunking uses a sliding window of $K{=}512$
tokens with stride $S{=}K$ unless otherwise specified.

\paragraph{Data and evaluation.}
Training uses the \textsc{GoodWiki-Long-Synthetic} \texttt{train} split
(10{,}000 query-disjoint queries; positives enter the loss at full
weight and distractors as $\lambda$-weighted partials, and the graded
labels also provide the test-time ranking signal documented in
\Cref{tab:goodwiki_long_synthetic_ir}). Reported
test numbers (\Cref{tab:gw_lcm_pareto}) are nDCG@10 on the 5{,}854-query
\texttt{test} split with the full 53{,}562-document corpus; self-matches
are removed. nDCG uses exponential graded gains
($2^{\mathrm{rel}}{-}1$), so the rephrased positive (score 2) carries
gain 3 and each topical distractor (score 1) gain 1; the auxiliary
precision/recall metrics in the released result files use linear
gains $\mathrm{rel}/2$. We apply the same protocol and metric
implementation to every system in \Cref{tab:gw_lcm_pareto} (sparse,
native long-context dense, bare-GN, and REIGN). All runs use a single
24\,GB consumer GPU.

\section{Training-Objective Ablation}\label{apx:loss_ablation}

\Cref{tab:loss_ablation} compares the three-way cosine
recipe of \Cref{ssec:objective} against InfoNCE variants under matched
conditions (\textsc{base-l3} on \textsc{GTE-small}, chunk/stride 512,
20 epochs unless noted, seed 42, best-validation selection). The
InfoNCE grid covers temperature $\tau \in \{0.07, 0.1\}$, distractors
included as graded soft positives ($\alpha{=}0.5$) or excluded
($\alpha{=}0$), batch size 18 vs.\ 48, a warm start from the released
checkpoint, and a 50-epoch run. The cosine recipe wins every
benchmark against every InfoNCE variant; its 20-epoch replica (66.99)
tracks the released 50-epoch checkpoint at the matched stride (67.09
@s512), so training length does not explain the gap. Within the InfoNCE family
the graded distractor signal does not help. The $\alpha{=}0$
arms beat their $\alpha{=}0.5$ counterparts on all three benchmarks
at $\tau{=}0.07$ and on two of three at $\tau{=}0.1$. Whether the
pairwise cosine objective would also gain from excluding
distractors is untested; the released recipe includes them as
$\lambda$-weighted partials. The same ordering holds on
short-text MTEB retrieval (\textsc{ArguAna}/\textsc{FiQA-2018} at
50.94/33.62 nDCG@10 for the 20-epoch cosine arm, best of the arms
measured on MTEB; the released checkpoint's MTEB numbers are in
\Cref{tab:mteb}).

\begin{table}[ht]
\centering
\caption{\textbf{Training-objective ablation} (\textsc{base-l3} on
\textsc{GTE-small}, chunk/stride 512, 20 epochs unless noted).
nDCG@10 on \textsc{GoodWiki-Long} test and LoCo (macro over 12
subtasks), nDCG@100 on \textsc{DAPFAM} test. The reference row reports the released
checkpoint's main-table values at their respective strides (GoodWiki
@s384; 67.09 at the matched s512; LoCo and DAPFAM @s512).\\
$\alpha$ = distractor (soft-positive) weight; bs = batch size;
\textbf{bold} = best ablation arm per column.\\
$^{\dagger}$Validation only degrades after the first validation
epoch; the row is the epoch-1 best-validation snapshot.\\
$^{\ddagger}$Warm start from the released checkpoint; validation
decays monotonically from the warm-start value, so the row is an
epoch-1 snapshot close to the released model.}
\label{tab:loss_ablation}
\scriptsize
\setlength{\tabcolsep}{3pt}
\begin{tabular}{lrrr}
\toprule
\textbf{Objective} & \textbf{GoodWiki} & \textbf{LoCo} & \textbf{DAPFAM} \\
\midrule
Three-way cosine, bs18 (recipe)         & \textbf{66.99} & \textbf{68.26} & \textbf{31.60} \\
Three-way cosine, bs48                  & 65.42 & 66.77 & 31.11 \\
\midrule
InfoNCE $\tau{=}.07$, $\alpha{=}0$, bs48        & 65.91 & 65.94 & 30.71 \\
InfoNCE $\tau{=}.07$, $\alpha{=}0$, bs48, 50 ep & 66.37 & 66.20 & 30.49 \\
InfoNCE $\tau{=}.07$, $\alpha{=}.5$, bs48       & 64.52 & 65.25 & 30.46 \\
InfoNCE $\tau{=}.07$, $\alpha{=}.5$, bs48, warm$^{\ddagger}$ & 65.74 & 65.88 & 30.87 \\
InfoNCE $\tau{=}.1$,\phantom{0} $\alpha{=}0$, bs48         & 64.76 & 64.95 & 30.64 \\
InfoNCE $\tau{=}.1$,\phantom{0} $\alpha{=}.5$, bs48        & 64.25 & 64.98 & 30.19 \\
InfoNCE $\tau{=}.07$, $\alpha{=}.5$, bs18$^{\dagger}$ & 60.16 & 63.15 & 29.48 \\
\midrule
\emph{Released ckpt (cosine, 50 ep)}    & \emph{67.31} & \emph{68.92} & \emph{31.69} \\
\bottomrule
\end{tabular}
\end{table}

\paragraph{Reproducibility note: InfoNCE needs a large negative
pool.} The InfoNCE path contrasts each anchor against the shifted
in-batch positives only, so at batch 18 the per-step pool is 324
pairs, versus 360 pairs including partials on the cosine path. At
$\tau{=}0.07$ this pool is too small for stable training. Validation nDCG@10
peaks at the first validation epoch (0.73) and decays monotonically
to 0.20 by epoch 19, so the batch-18 row above is an
epoch-1 snapshot. Batch 48 (2{,}304 pairs) removes the collapse; the
$\alpha{=}0$ arm improves steadily (0.76$\,\to\,$0.78), while the
$\alpha{=}0.5$ arm still decays after an early peak
(0.76$\,\to\,$0.61). The cosine objective shows no such sensitivity,
and its validation improves monotonically at both batch sizes.

\section{DAPFAM Protocol Details}\label{apx:dapfam}

\paragraph{Dataset.} \textsc{DAPFAM} \citep{ayaou2025dapfam} pairs
1{,}247 query patent families with a corpus of 45{,}336 candidate
families, each represented by the concatenated text of its constituent
patent documents. Relevance is binary, averaging 19.99 cited
positives and 20 provided random non-cited negatives per query.
Sampled whitespace-token length statistics give queries a median of
12{,}330 tokens (p75~21{,}204, p95~60{,}616, max~485{,}712) and corpus
documents a median of 7{,}346 tokens (p75~12{,}262, p95~31{,}952,
max~294{,}417). We use the \textsc{FullText} view throughout, which
is the long-document setting REIGN targets.

\paragraph{Split construction.} \textsc{DAPFAM} ships eval-only.
To train REIGN on patents we built a query-disjoint train/val/test
partition: (i)~stratify the 1{,}247 query patents by quartiles of
positive count; (ii)~split randomly $\sim$70/15/15 with seed 42,
query-disjoint (no query appears in more than one split);
(iii)~scope the original IN-/OUT-IPC partition to the held-out
test set, so \textsc{test\_in} and \textsc{test\_out} are
positives-only subsets of \textsc{test} and the cross-domain
breakdown is computed on identical queries across systems.
The resulting counts are 872/188/187 unique queries
(34{,}869/7{,}520/7{,}480 qrels rows) for train/val/test,
with 180 and 138 queries in \textsc{test\_in} and
\textsc{test\_out} respectively.

\paragraph{Training protocol.} REIGN is fine-tuned on \textsc{DAPFAM}
via a standard query/positive/negative contrastive setup with InfoNCE
and false-negative masking. Binary qrels disable the objective's
partial branch (\texttt{partial-policy}~$=$~\texttt{ignore} in the
released configuration). Negatives are
\textsc{DAPFAM}'s provided random $\texttt{score}{=}0$ families
(\texttt{n-negatives-per-sample}~$=$~4 plus in-batch negatives;
using all 20 provided negatives exceeds 24\,GB of memory at FullText
sequence length). Warm-start runs initialise from the
\textsc{GoodWiki-Long}-trained REIGN checkpoint; cold-start runs
initialise the REIGN encoder from scratch with the same frozen GN.
Optimiser AdamW, lr~$\in\{10^{-5}, 5{\cdot}10^{-6}, 2{\cdot}10^{-6},
10^{-6}\}$, weight decay~$\in\{10^{-4}, 10^{-2}, 10^{-1}\}$,
cosine schedule, temperature 0.07, 16-bit mixed precision,
seed 42, validation every 3 epochs.

\paragraph{Evaluation protocol.} The authoritative metric is
nDCG@100 computed with a qrels-based evaluator over the full
45{,}336-document FullText corpus at top-$k$~$=$~100; the
in-training validation nDCG@10 is an in-batch proxy only and is
\emph{not} the reported number. Self-matches are removed before
scoring: for typically $\sim$12--16 of the 187 test queries the query
document itself appears in the corpus, and those corpus copies are
dropped from the candidate pool. The same protocol is applied
identically to zero-shot REIGN, fine-tuned REIGN, and every sparse
or dense baseline; only the encoder varies.

\paragraph{Significance re-runs.} Two deltas in
\Cref{tab:dapfam_significance} come from re-runs. Jina-v3 scores
32.98 against the main-table 32.97 (fp16 non-determinism) and
Stella-1.5B 33.00 against 32.91 (batch 8 rather than the main-table
batch 4); every other row matches \Cref{tab:dapfam_headline} exactly,
and holding Stella at its main-table 32.91 leaves the verdict
unchanged ($p{=}.843$).

\end{document}